\documentclass[hidelinks,onefignum,onetabnum]{siamart251216}

\usepackage{lipsum}
\usepackage{amsfonts}
\usepackage{graphicx}
\usepackage{epstopdf}
\usepackage[sort,numbers]{natbib}
\usepackage{algorithmic}
\usepackage{bm}
\ifpdf
  \DeclareGraphicsExtensions{.eps,.pdf,.png,.jpg}
\else
  \DeclareGraphicsExtensions{.eps}
\fi

\newsiamremark{remark}{Remark}
\newsiamremark{hypothesis}{Hypothesis}
\crefname{hypothesis}{Hypothesis}{Hypotheses}
\newsiamthm{claim}{Claim}
\newsiamremark{fact}{Fact}
\crefname{fact}{Fact}{Facts}

\headers{Neural Flow Operators}{S. CHEN, J. HE , AND X.-C. TAI}

\title{Neural Flow Operators Can Approximate Any Operator: Abstract Frameworks and Universal Approximations \thanks{Submitted to the editors DATE.
}
}

\author{Shuang Chen\thanks{Qiuzhen College, Tsinghua University, Haidian District, Beijing 100084, China (\email{chenshuang24@mails.tsinghua.edu.cn}).}
\and Juncai He\thanks{Yau Mathematical Sciences Center, Tsinghua University, Haidian District, Beijing 100084, China (\email{jche@tsinghua.edu.cn}), corresponding author.}
\and Xue-Cheng Tai\thanks{Norwegian Research Center, Fantoftvegen 38, 5072 Bergen, Norway (\email{xtai@norceresearch.no}), corresponding author.}
}

\usepackage{amsopn}

\ifpdf
\hypersetup{
  pdftitle={NeuralFlow},
  pdfauthor={S. Chen, J. He, and X.-C. Tai}
}
\fi

\begin{document}

\maketitle

\begin{abstract}
We introduce an abstract neural flow framework for neural networks and neural operators. The framework contains two continuous-depth models, namely neural flows with composition and separation structures, and covers both finite-dimensional function approximation and infinite-dimensional operator approximation. We prove well-posedness and universal approximation properties for the corresponding neural flows, including, to the best of our knowledge, the first universal approximation result for flow-based models between infinite-dimensional spaces. We also obtain universal approximation results for convolutional neural flow models. Through suitable time discretizations, the composition structure recovers ResNet-type architectures, while the separation structure, via a splitting-based discretization, yields plain architectures. This gives a unified flow-based route to both residual and plain architectures for neural networks and neural operators with fully connected or convolutional linear layers.
\end{abstract}



\section{Introduction}
Deep learning architectures such as fully connected deep neural networks (DNNs) \cite{lecun2015deep,bengio2017deep}, convolutional neural networks (CNNs) \cite{lecun1998gradient,krizhevsky2012imagenet,he2016deep}, and, more recently, neural operators \cite{lu2021learning,li2021fourier} have achieved remarkable empirical success across computer vision, natural language processing, and scientific computing, as well as in data-driven modeling of complex systems. This success has led to a rapidly growing mathematical literature on expressive power, approximation properties, and architectural design. Despite substantial progress, however, much of the existing theory remains fragmented: most available results are developed for specific model classes, rely on architecture-dependent constructions, and are not easily transferred across finite-dimensional neural networks and operator-learning models on function spaces.

For standard feedforward neural networks on finite-dimensional Euclidean spaces, the approximation theory is already extensive and well developed. For shallow (one-hidden-layer) neural networks, a substantial body of work has established qualitative universal approximation results~\cite{cybenko1989approximation,hornik1989multilayer,leshno1993multilayer,mhaskar1993approximation}, quantitative approximation estimates~\cite{barron1993universal,mhaskar1995degree,klusowski2018approximation}, rates in reproducing kernel Hilbert space and Barron-space settings~\cite{bach2017breaking,e2019barron,e2019priori}, and, more recently, sharp optimal approximation rates~\cite{siegel2024sharp,liu2025integral}. In comparison, deep neural networks exhibit substantially stronger expressivity and approximation properties due to the presence of depth. On the expressivity side, a large literature has quantified the advantages of deep networks through piecewise linear representations~\cite{montufar2014number,telgarsky2016benefits,mhaskar2016deep,lu2017expressive,he2020relu,he2022relu} and, more generally, piecewise polynomial representations~\cite{opschoor2020deep,he2023expressivity,he2023deep}. On the approximation side, a sequence of superconvergence-type results~\cite{yarotsky2017error,shen2019deep,shen2022optimal,yang2023nearly,siegel2023optimal,yang2025deep}, building on the bit-extraction technique~\cite{bartlett1998almost,bartlett2019nearly}, has demonstrated a significant separation between the approximation power of deep and shallow networks.

For CNNs, the approximation theory is less unified and remains more closely tied to architectural details. Existing analyses typically exploit specific features of convolution, such as locality, sparsity, and parameter sharing~\cite{zhou2020universality,bao2023approximation,li2024approximation}, together with structural assumptions on the target functions~\cite{petersen2020equivalence,oono2019approximation}. More recently, general quantitative approximation results for deep CNNs with two-dimensional inputs were established in~\cite{he2022approximation} and further extended to Korobov spaces in~\cite{fang2025two}. Nevertheless, most available proofs remain model-specific and do not by themselves provide a general framework that simultaneously encompasses fully connected networks and operator-learning architectures.

For neural operators on infinite-dimensional or more general abstract vector spaces, a unifying theory is even more needed. Neural operators are designed to learn maps between function spaces and have emerged as effective models for approximating solution operators of partial differential equations~\cite{lu2021learning,li2021fourier,cao2021choose,kovachki2023neural,raonic2023convolutional,liu2024mitigating,li2026deep} as well as other infinite-dimensional nonlinear maps~\cite{chen1995universal,molinaro2023neural,gao2024adaptive,wu2024transolver}. On the theoretical side, universal approximation and quantitative error bounds have been established for several important architectures, including Fourier neural operators (FNO)~\cite{kovachki2021universal} and DeepONet~\cite{lanthaler2022error}. In addition, \cite{he2024mgno,he2025self} proposed an abstract infinite-dimensional feedforward framework covering a broad class of neural operator architectures and established corresponding qualitative approximation results. Nevertheless, most available analyses remain architecture-dependent, relying in essential ways on the particular lifting, kernel representation, or spectral parameterization used by the model. As a result, there is still no broadly applicable framework that explains, within a single formulation, why both neural networks and neural operators possess universal approximation properties and how these properties are connected.

A natural bridge among these model classes is the continuous-depth viewpoint, namely, the interpretation of deep architectures through flow-type models. Distinct from deterministic flow-based generative models~\cite{lipman2023flow,guo2022normalizing,geng2025mean}, continuous or dynamical-system descriptions of deep learning view layer composition as the evolution of an underlying flow~\cite{e2017proposal,chang2018multilevel,haber2018learning,li2017flow,lu2018beyond,chen2018neuralode}, rather than merely as a finite sequence of discrete transformations. This perspective is particularly attractive from the standpoint of numerical analysis, since it allows depth to be regarded as a discretization parameter and the underlying continuous model to be treated as the primary object. Recent theoretical advances have rigorously characterized the approximation and interpolation properties of such flow-based models from a control-theoretic viewpoint. In particular, universal approximation has been established for continuous-time residual networks under mild structural assumptions~\cite{li2023deep,tabuada2024universal,duan2025minimal,masato2025universal}, while constructive control strategies have been developed to ensure both universal approximation and simultaneous classification of multiple data distributions~\cite{ruizbalet2023neuralode}. Closely related results on universal interpolation have also been obtained, revealing strong connections between interpolation capability and ensemble controllability of the underlying dynamical systems~\cite{cuchiero2020deep,cheng2025interpolation,cai2025achieving,duan2025minimal}. Moreover, the interplay among width, depth, and flow-based constructions has been used to study sharp width requirements for universal approximation~\cite{li2023minimum,cai2023achieve,alvarezlopez2024interplay}. Beyond ODE-based formulations, simultaneous depth-width limits have also led to PDE/integro-differential descriptions of learning dynamics~\cite{markowich2024pde}, and even attention-based architectures such as Transformers admit related continuous formulations as discretizations of structured integro-differential systems~\cite{tai2024pottsmgnet,tai2025mathematical}. However, most existing continuous-depth analyses remain primarily finite-dimensional and do not provide a unified approximation framework covering both classical neural networks and neural operators.

In this work, we develop a unified abstract neural flow framework for both neural networks and neural operators. A central feature of the framework is the introduction of two distinct continuous-depth models: a neural flow with composition structure and a neural flow with separation structure. The separation-structured neural flow was first used in Potts-MgNet~\cite{tai2024pottsmgnet} to give a mathematical explanation of neural networks in the continuous and discrete setting. The formulation in this paper is partly inspired by this idea and is further developed into a unified abstract framework that also includes separation-structured flows and infinite-dimensional operator settings. The proposed setting covers finite-dimensional flows for function approximation and infinite-dimensional flows for operator approximation within a common formulation. Crucially, the framework is formulated at the level of abstract Hilbert spaces 
and is not restricted to function spaces: the approximation results hold for 
arbitrary pairs of Hilbert spaces, regardless of whether they arise as spaces 
of functions defined on a domain. Within this framework, we establish well-posedness of the associated flow models and prove universal approximation properties for the resulting neural flows. 
To the best of our knowledge, this is the first work to extend the neural flow perspective to infinite-dimensional operator approximation together with corresponding universal approximation results. By further restricting the admissible linear transformations to convolutional operators, we also obtain universal approximation results for convolutional neural flow models.

The distinction between the two flow structures becomes particularly important at the discrete level. The composition structure recovers, through explicit time discretization, the standard ResNet-type architectures that are commonly associated with continuous-depth models. In contrast, the separation structure, combined with a suitable splitting-based discretization, yields plain feedforward architectures. In this way, the same framework leads naturally to approximation results for finite-depth neural networks and neural operators with either fully connected or convolutional linear layers, under both plain and ResNet-type architectures. To the best of our knowledge, previous flow-based analyses have mainly focused on residual-type models, while the present framework provides a unified route to both residual and plain architectures in finite- and infinite-dimensional settings.

From the standpoint of numerical analysis, the neural flow may be viewed as the underlying dynamical system, while concrete architectures arise as its discretizations. This perspective provides a common analytical framework for understanding the approximation properties of several important classes of learning models, and clarifies how different architectures can emerge from different continuous-depth structures and discretization procedures.

The remainder of the paper is organized as follows. In Section~\ref{sec:notation}, we introduce the notation and formulate the abstract neural flow framework. Section~\ref{sec:classicmodels} shows how this framework recovers classical feedforward and convolutional neural networks and neural operators. Section~\ref{sec:flowUAP} is devoted to well-posedness and universal approximation results for neural flows. In Section~\ref{sec:application1}, we establish approximation results for convolutional neural flow models. Section~\ref{sec:application2} derives approximation results, via suitable time discretizations, for finite-depth neural networks and neural operators under both plain and ResNet-type architectures, in both fully connected and convolutional settings. Finally, Section~\ref{sec:conclusions} concludes with a discussion of limitations, possible generalizations, and future directions.

\section{Our proposed neural flow operators}\label{sec:notation}

In this section, we introduce the notation and formulate the abstract neural flow framework studied throughout the paper. The framework is built around two 
continuous-depth flow models, which we call the composition-structured and 
separation-structured neural flows. A central goal is to work at the level of 
abstract Hilbert spaces from the outset, so that the resulting theory applies 
simultaneously to finite-dimensional function approximation and 
infinite-dimensional operator approximation, without requiring separate 
treatments for neural networks and neural operators.

\subsection{Abstract neural flow models}

We describe two representative classes of neural flow models and explain how 
their suitable time discretizations recover finite-depth plain and ResNet-type 
architectures for both neural networks and neural operators. The two flow 
structures play complementary roles: the composition structure leads naturally, 
via explicit time discretization, to ResNet-type architectures, while the 
separation structure, combined with a splitting-based discretization, yields 
plain feedforward architectures. This distinction at the continuous level, therefore, provides a unified and principled explanation for two of the most widely used architectural families in deep learning.

The central object of approximation is a map
\begin{equation*}
    \mathcal{O} \colon \mathcal{C} \subset \mathcal{X} \to \mathcal{Y},
\end{equation*}
where $\mathcal{X}$ and $\mathcal{Y}$ are Hilbert spaces over $\mathbb{R}$, 
and $\mathcal{C}$ is a compact subset of $\mathcal{X}$. No further structural 
assumptions are imposed on $\mathcal{X}$ and $\mathcal{Y}$: they need not be 
spaces of functions defined on a domain, and the framework therefore covers 
settings beyond classical operator learning. The goal is to construct abstract 
neural flow operators that approximate $\mathcal{O}$ to arbitrary accuracy, 
and to identify the conditions under which such an approximation is possible.

The abstract flow models we shall propose evolve in a latent space. To incorporate this,  we introduce an abstract latent space $\mathcal Z$, which is a separable  Hilbert space over $\mathbb R$. Typical choices include finite‑dimensional Euclidean spaces $\mathbb R^d$, matrix spaces $\mathbb R^{d\times d'}$, or function spaces such as Sobolev spaces $H^s(\Omega)$. For a given $D \in \mathbb N^+$, we write
\begin{equation*}
    \mathcal Z^D := \underbrace{\mathcal Z \times \mathcal Z \times \cdots \times \mathcal Z}_{D}
\end{equation*}
for the $D$-fold product space.

The evolution of the flow is described by a differential equation on $\mathcal Z^D$. In particular, if $\mathcal Z$ is finite‑dimensional, the flow is governed by an ordinary differential equation; if $\mathcal Z$ is a function space, the flow may be described by a partial differential equation. In abstract form, we define the neural flow on $\mathcal Z^D$ as
\begin{equation*}\label{eq:defNF}
{}_{\mathbf n}\mathbf{F}(\mathcal Z^D):= \left\{ \mathcal F^T_{\theta_t}: \mathcal Z^D \to \mathcal Z^D
   ~\left|~ \begin{cases}
        &\frac{\mathrm{d}  z}{\mathrm{d} t} = \Phi_{\theta_t}( z), \quad 0 < t \le T, \\
        &\mathcal F^T_{\theta_t}~:~  z(0) \mapsto  z(T) \in \mathcal{Z}^D.
    \end{cases} \right.
\right\}    
\end{equation*}
where $\Phi_{\theta_t}$ denotes the nonlinear vector field parameterized by $\theta_t$ for $t\in(0,T]$. Given two bounded linear maps
\begin{equation*}
    \mathcal P \in \mathcal L(\mathcal X,\mathcal Z^D)
    \quad \text{and} \quad
    \mathcal R \in \mathcal L(\mathcal Z^D,\mathcal Y),
\end{equation*}
we define the corresponding {\bf abstract neural flow} model from $\mathcal X$ to $\mathcal Y$ by
\begin{equation*}
    \left\{
    \mathcal R \circ \mathcal F^T_{\theta_t} \circ \mathcal P : \mathcal X \to \mathcal Y
    ~\middle|~
    \mathcal F^T_{\theta_t} \in {}_{\mathbf n}\mathbf{F}(\mathcal Z^D)
    \right\}.
\end{equation*}

Motivated by two different ways of modeling the nonlinear vector field $\Phi_{\theta_t}$, we focus on the following two representative classes.

\begin{definition}\label{def:nF}
Let $\sigma,\psi:\mathcal Z\to\mathcal Z$ be nonlinear maps, extended to $\mathcal Z^D$ component  -wise.

The {\bf neural flow with composition structure} is defined by
\begin{equation}\label{eq:CompScheme}
    {}_{\mathbf n}\mathbf{F}^{\mathbf c}(\mathcal Z^D) := 
    \left\{ \mathcal F^T_{\theta_t}
   ~\left|~ \begin{cases}
        &\frac{\mathrm{d}  z}{\mathrm{d} t} = \Phi_{\theta_t}^c(z) := \sigma(\mathcal W_t  z + \bm b_t), ~ 0 < t \le T, \\
        &\mathcal F^T_{\theta_t}~:~  z(0) \mapsto  z(T) \in \mathcal{Z}^D,
    \end{cases} \right.
\right\},    
\end{equation}
where
\begin{equation*}
\theta_t = \{\mathcal W_t \in \mathcal L(\mathcal Z^D,\mathcal Z^D),\ \bm b_t \in \mathcal Z^D\}.
\end{equation*}

The {\bf neural flow with separation structure} is defined by
\begin{equation}\label{eq:SepScheme}
    {}_{\mathbf n}\mathbf{F}^{\mathbf s}(\mathcal Z^D) :=
    \left\{ \mathcal F^T_{\theta_t}
   ~\left|~ \begin{cases}
        &\frac{\mathrm{d}  z}{\mathrm{d} t} = \Phi_{\theta_t}^s(z) := \mathcal W_t  z + \bm b_t + \alpha_t \psi( z), ~ 0 < t \le T, \\
        &\mathcal F^T_{\theta_t}~:~  z(0) \mapsto  z(T) \in \mathcal{Z}^D,
    \end{cases} \right.
\right\},    
\end{equation}
where
\begin{equation*}
    \theta_t = \{\mathcal W_t \in \mathcal L(\mathcal Z^D,\mathcal Z^D),\ \bm b_t \in \mathcal Z^D,\ \alpha_t \in \mathbb R\}.
\end{equation*}
\end{definition}

\begin{remark}
Most of the neural flow studies in the literature have focused on the composition structure in \eqref{eq:CompScheme}, see 
\cite{markowich2024pde,li2026deep,siegel2024sharp,shen2022optimal} which is closely related to the standard ResNet-type architectures. In contrast, the separation structure in \eqref{eq:SepScheme} has received much less attention, and its approximation properties remain largely unexplored. 
    The separation-structured neural flow in \eqref{eq:SepScheme} was first introduced in Potts-MgNet~\cite{tai2024pottsmgnet}, where continuous and discrete viewpoints were used to explain encoder-decoder-type neural networks. The present work is partly motivated by that perspective, but develops it further into a unified abstract framework that also includes both composite and separation-structured flows,  and infinite-dimensional operator settings. In particular, the current paper provides a systematic theoretical analysis of neural flow models, including well-posedness and universal approximation results.
\end{remark}

\subsection{Time discretization of neural flow models to recover conventional finite-depth neural architectures}
The continuous-depth neural flow models introduced above provide the underlying mathematical framework, but practical neural network architectures are discrete, consisting of finitely many layers with prescribed transformations. The purpose of this subsection is to make this connection precise. We show that suitable time-discretization schemes applied to neural flows with composition and separation structures naturally and transparently produce finite-depth architectures commonly used in deep learning.

More specifically, explicit time discretization of the composition-structured flow recovers ResNet-type architectures, in which each layer is given by a residual update. In contrast, a splitting-based discretization of the separation-structured flow yields plain feedforward architectures, in which no skip connections are present. This correspondence applies to both fully connected and convolutional linear layers, and to both neural networks in finite-dimensional spaces and neural operators in infinite-dimensional settings. From the viewpoint of numerical analysis, the network depth $L$ plays the role of the number of time steps, while the step size $\Delta t=T/L$ determines the resolution of the discretization. In this way, depth is interpreted not merely as an architectural hyperparameter, but as a discretization parameter for the underlying continuous flow.

To make this precise, let
\begin{equation*}
    \Delta t = \frac{T}{L},
    \qquad
    t_\ell = \ell \,\Delta t,
    \qquad
    \ell = 0, 1, \dots, L,
\end{equation*}
be a uniform partition of $[0, T]$ into $L$ subintervals of equal length.

For the {\bf neural flow with composition structure}, we discretize it by  the explicit Euler method:
\begin{equation*}
    \frac{ z(t_{\ell+1})- z(t_\ell)}{\Delta t}
    =
    \sigma\bigl(\mathcal W_{t_\ell} z(t_\ell)+\bm b_{t_\ell}\bigr).
\end{equation*}
Writing
\begin{equation*}
     z^\ell :=  z(t_\ell),
    \qquad
    \mathcal W^\ell := \mathcal W_{t_{\ell-1}},
    \qquad
    \bm b^\ell := \bm b_{t_{\ell-1}},
\end{equation*}
This discretization gives us the finite-depth model
\begin{equation}\label{eq:DisNFC}
\left\{
\mathcal F:\bm v \mapsto \bm u
~\middle|~
\begin{array}{l}
z^0 = \mathcal P \bm v \in \mathcal Z^D,\\
 z^\ell =  z^{\ell-1} + \Delta t\, \sigma\bigl(\mathcal W^\ell  z^{\ell-1} + \bm b^\ell\bigr)
\in \mathcal Z^D,\quad \ell=1:L,\\
\bm u = \mathcal R  z^L \in \mathcal Y,
\end{array}
\right\}
\end{equation}
which recovers the ResNet-type architecture.

For the {\bf neural flow with separation structure}, we adopt the following semi-implicit discretization:
\begin{equation*}\label{semi-implicit discretization}
\begin{cases}
    \displaystyle
    \frac{ z^{\ell-\frac12}- z^{\ell-1}}{\Delta t}
    = \mathcal W^\ell   z^{\ell-1} + \bm b^\ell, \\[2ex]
    \displaystyle
    \frac{  z^\ell-  z^{\ell-\frac12}}{\Delta t}
    = \alpha^\ell \psi(  z^\ell),
\end{cases},\qquad \ell=1:L.
\end{equation*}
Equivalently,
\begin{equation*}
\begin{cases}
    \displaystyle
      z^{\ell-\frac12}
    = (I+\Delta t\mathcal W^\ell)   z^{\ell-1} +\Delta t \bm b^\ell
    =\widetilde{\mathcal W}^\ell   z^{\ell-1} +\widetilde{\bm b}^\ell,\\
    \displaystyle
      z^\ell-\Delta t\alpha^\ell \psi(  z^\ell)
    =   z^{\ell-\frac12},
\end{cases} \qquad \ell=1:L.
\end{equation*}
For each $\ell=1:L$, we assume that there exists a nonlinear map $\sigma^\ell:\mathcal Z \to \mathcal Z$ such that
\begin{equation}\label{eq:implicit-sigma}
    \sigma^\ell\bigl(  z - \Delta  t \alpha^\ell \psi(  z)\bigr) =   z,
\end{equation}
then the second step can be written explicitly as $z^\ell= \sigma^\ell(z^{\ell-1/2})$, and thus we obtain
\begin{equation}\label{eq:DisNFS}
\left\{
\mathcal F:\bm v \mapsto \bm u
~\middle|~
\begin{array}{l}
  z^0 = \mathcal P \bm v \in \mathcal Z^D,\\
  z^\ell = \sigma^\ell\bigl(\widetilde{\mathcal W}^\ell   z^{\ell-1} + \widetilde{\bm b}^\ell\bigr)
\in \mathcal Z^D,\quad \ell=1,\dots,L,\\
\bm u = \mathcal R   z^L \in \mathcal Y.
\end{array}
\right\}
\end{equation}
This yields a plain feedforward architecture.

\begin{remark}
In this work, we use three nonlinear functions: $\sigma$ for both the composition-structured flow and its discretization, $\psi$ for the separation-structured flow, and $\sigma^\ell$ for the semi-implicit discretization of the neural flows. Although $\sigma^\ell$ is allowed to vary with the layer index $\ell$, Section~\ref{sec:activation} shows that it may be chosen uniformly across all layers. In fact,  commonly used activation functions $\sigma$, related to $\sigma^\ell$ when it is independent of $\ell$, satisfy the implicit relation~\eqref{eq:implicit-sigma} for a suitable choice of $\psi$. See Sections~\ref{sec:activation} and~\ref{sec:application2} for details.
\end{remark}

\section{Our model recovers classical neural networks and operators}\label{sec:classicmodels}
At this stage, the latent space $\mathcal{Z}$, the activation map $\sigma$ appearing in~\eqref{eq:DisNFC}, and the layer-dependent activation map $\sigma^\ell$ appearing in~\eqref{eq:DisNFS} are kept deliberately abstract. 
This level of generality is intentional: by leaving $\mathcal{Z}$ unspecified, the framework simultaneously encompasses finite-dimensional Euclidean spaces, function spaces such as $H^s(\Omega)$, and more general Hilbert spaces that need not arise as spaces of functions on a domain. Different choices of $\mathcal{Z}$ and of the linear operators $\mathcal{L}$ and $\mathcal{R}$ then yield different concrete architectures, while the underlying flow structure and the associated approximation theory remain the same.

We now make this correspondence explicit by showing how the abstract neural flow models introduced above recover several classical architectures as particular instances. Specifically, choosing $\mathcal{Z}$ to be a finite-dimensional Euclidean space and the linear operators to be standard matrix multiplications recovers fully connected and convolutional deep neural networks; choosing $\mathcal{Z}$ to be an infinite-dimensional function space and the linear operators to be bounded operators between such spaces recovers neural operator architectures. We follow the convention in the literature and use the term \emph{neural networks} to refer to finite-dimensional models and \emph{neural 
operators} to refer to infinite-dimensional models, although both are special cases of the same abstract framework and are treated on an equal footing throughout this work.

\subsection{Neural flow networks and deep (finite-depth) fully connected neural networks}
We first show that classical neural networks are covered by our proposed abstract models. Consider the finite-dimensional setting
\[
\mathcal X = \mathbb R^{d_x}, \qquad \mathcal Y = \mathbb R^{d_y}, \qquad \mathcal Z = \mathbb R,
\]
with nonlinear  functions $\sigma,\psi:\mathbb R\to\mathbb R$. In this case, the abstract neural flow framework reduces to a flow map on $\mathbb R^D$. More precisely, the class of {\bf neural flow networks} with the \emph{composition structure} is given by
\begin{equation}\label{eq:nFcRD}
{}_{\mathbf n}\mathbf{F}^{\mathbf c}(\mathbb R^D):= \left\{  F^T_{\theta_t}: \mathbb R^D \to \mathbb R^D
   ~\left|~ \begin{cases}
        &\frac{\mathrm{d}  z}{\mathrm{d} t} = \sigma(W_t z + b_t), \quad 0 < t \le T, \\
        & F^T_{\theta_t}~:~  z(0) \mapsto  z(T) \in \mathbb R^D,
    \end{cases} \right.
\right\},
\end{equation}
and the class with the \emph{separation structure} is given by
\begin{equation}\label{eq:nFsRD}
{}_{\mathbf n}\mathbf{F}^{\mathbf s}(\mathbb R^D):= \left\{  F^T_{\theta_t}: \mathbb R^D \to \mathbb R^D
   ~\left|~ \begin{cases}
        &\frac{\mathrm{d}  z}{\mathrm{d} t} = W_t z + b_t +\alpha_t \psi(z), \quad 0 < t \le T, \\
        & F^T_{\theta_t}~:~  z(0) \mapsto  z(T) \in \mathbb R^D,
    \end{cases} \right.
\right\},
\end{equation}
where $W_t \in \mathbb R^{D\times D}$, $b_t \in \mathbb R^D$, and $\alpha_t\in\mathbb R$ for all $0<t\le T$.

It follows directly from the discretizations \eqref{eq:DisNFC} and \eqref{eq:DisNFS} that classical finite-depth fully connected neural networks, c.f.~\cite{bengio2017deep}, arise as discrete realizations of these neural flow models. More precisely, the explicit discretization of the composition structure \eqref{eq:nFcRD} yields ResNet-type fully connected networks with activation $\sigma$ \cite{he2016deep}. On the other hand, if for each layer $\ell$ the map $\sigma^\ell:\mathbb R\to\mathbb R$ satisfies \eqref{eq:implicit-sigma} with respect to $\psi$ in \eqref{eq:nFsRD}, then the semi-implicit discretization of the separation structure yields plain fully connected networks with activations $\sigma^\ell$.

Accordingly, we define the class of {\bf neural flow networks} from $\mathbb R^{d_x}$ to $\mathbb R^{d_y}$ by
\begin{equation*}
\left\{\left.  R \circ  F^T_{\theta_t} \circ  P : \mathbb R^{d_x} \to \mathbb R^{d_y} ~\right|~  F^T_{\theta_t} \in {}_{\mathbf n}\mathbf{F}^{\mathbf c}(\mathbb R^D)  \text{ or } {}_{\mathbf n}\mathbf{F}^{\mathbf s}(\mathbb R^D) \right\},
\end{equation*}
where $P \in \mathcal L(\mathbb R^{d_x}, \mathbb R^D) \cong \mathbb R^{D \times d_x}$ and $R \in \mathcal L(\mathbb R^D,\mathbb R^{d_y})\cong \mathbb R^{d_y \times D}$. Their finite-depth discretizations recover the standard classes of fully connected feedforward networks from $\mathbb R^{d_x}$ to $\mathbb R^{d_y}$, including both ResNet-type and plain architectures with width $D$.

Universal approximation properties for deep neural networks have been established separately for ResNet-type architectures~\cite{lin2018resnet,oono2019approximation,masato2025universal,liu2024characterizing} and for plain architectures~\cite{yarotsky2017error,shen2019deep,shen2022optimal,yang2023nearly,siegel2023optimal,yang2025deep}. Our framework provides a unified derivation of these results and, moreover, clarifies how residual and plain architectures emerge from underlying flow models through different discretization mechanisms.

\subsection{Neural flow operators and deep neural operators}
We next show that the abstract framework also covers many neural operator architectures studied in the literature~\cite{lu2021learning,li2021fourier,kovachki2023neural,raonic2023convolutional,liu2024mitigating,he2024mgno,li2026deep}. Consider the infinite-dimensional setting, where $\Omega \subset \mathbb R^d$ is a bounded domain and both $\mathcal X$ and $\mathcal Y$ are Hilbert function spaces on $\Omega$. In this case, it is natural to choose the latent space $\mathcal Z$ to be a Hilbert function space on $\Omega$ or one of its subspaces. For simplicity, throughout this work we assume that $\mathcal Z \subseteq H^s(\Omega)$ for some $0 \le s \le 1$.

The nonlinear activation is then induced by pointwise function composition, namely
\begin{equation}\label{eq:operatorAct}
\begin{aligned}
    \sigma: H^{s}(\Omega) &\to H^s(\Omega), \\
      z(x,t) &\mapsto \sigma\bigl(z(x,t)\bigr),
\end{aligned}
\end{equation}
and similarly for $\psi$ and $\sigma^\ell$. In the neural-operator literature~\cite{li2021fourier,kovachki2023neural}, this is often interpreted as an elementwise activation. Here, however, it is viewed more naturally as a nonlinear operator acting on the function space $\mathcal Z$.

In this setting, the abstract neural flow framework reduces to flow maps on $\mathcal Z^D$. More precisely, the classes of {\bf neural flow operators} with \emph{composition structure} and \emph{separation structure} are given by
\begin{equation*}
    {}_{\mathbf n}\mathbf{F}^{\mathbf c}(\mathcal Z^D)
    \quad\text{and}\quad
    {}_{\mathbf n}\mathbf{F}^{\mathbf s}(\mathcal Z^D),
\end{equation*}
which are defined exactly as in Definition~\ref{def:nF}, with the function space $\mathcal Z$ and the induced nonlinear activation on $\mathcal Z$ introduced in \eqref{eq:operatorAct}.

Accordingly, we define the class of {\bf neural flow operators} from $\mathcal X$ to $\mathcal Y$ by
\begin{equation*}
   \left\{\left.  \mathcal R \circ  \mathcal F^T_{\theta_t} \circ  \mathcal P : \mathcal X \to \mathcal Y ~\right|~  \mathcal F^T_{\theta_t} \in {}_{\mathbf n}\mathbf{F}^{\mathbf c}(\mathcal Z^D)  \text{ or } {}_{\mathbf n}\mathbf{F}^{\mathbf s}(\mathcal Z^D) \right\},
\end{equation*}
where $\mathcal P \in \mathcal L(\mathcal X, \mathcal Z^D)$ and $\mathcal R \in \mathcal L(\mathcal Z^D,\mathcal Y)$.

It follows directly from the discretizations \eqref{eq:DisNFC} and \eqref{eq:DisNFS} that standard finite-depth neural operators arise as discrete realizations of these neural flow operators. More precisely, the explicit discretization of the composition structure ${}_{\mathbf n}\mathbf{F}^{\mathbf c}(\mathcal Z^D)$ yields ResNet-type neural operators with activation $\sigma$. On the other hand, if for each layer $\ell$ the map $\sigma^\ell:\mathcal Z\to\mathcal Z$ satisfies \eqref{eq:implicit-sigma} with respect to $\psi$ in ${}_{\mathbf n}\mathbf{F}^{\mathbf s}(\mathcal Z^D)$, then the semi-implicit discretization of the separation structure yields plain neural operators with activations $\sigma^\ell$.

Moreover, as discussed in detail in~\cite{he2024mgno}, most popular neural operator architectures are of the plain type. The present framework recovers many such models~\cite{lu2021learning,li2021fourier,cao2021choose,kovachki2023neural,raonic2023convolutional,liu2024mitigating,li2026deep} through appropriate choices of the latent space $\mathcal Z$ and of the parameterization of the linear operators $\mathcal W^\ell \in \mathcal L(\mathcal Z,\mathcal Z)$.

\subsection{Convolutional neural (flow) networks and operators}
Here, we show that by restricting the admissible linear operators to convolutional forms, we can obtain convolutional neural flow models, whose discretizations recover standard convolutional neural networks and neural operators.
Consider the case where the admissible linear operators are restricted to a convolutional form. In the finite-dimensional setting, one typically has $\mathcal X = \mathbb R^{d_x \times d_x \times c}$, while $\mathcal Y = \mathbb R^{k}$ for image classification and $\mathcal Y = \mathbb R^{d_y \times d_y \times c}$ for image segmentation. We take
$$\mathcal Z = \mathbb R^{d \times d} \quad \text{and} \quad \text{element-wise nonlinear activation }\sigma: \mathbb R^{d \times d} \to \mathbb R^{d \times d}
$$
In addition, we restrict each block operator $\left[\mathcal W_t\right]_{i,j}$ to be a convolution operator of kernel size $(2k+1)\times(2k+1)$, namely
\begin{equation*}
\left[\mathcal W_t \right]_{i,j}: \mathbb R^{d \times d} \to \mathbb R^{d \times d} \text{ is a convolution with size } (2k+1)\times(2k+1),
\end{equation*}
and assume that $
\left[\bm b_t \right]_i \in \{ a \bm 1 ~|~ a \in \mathbb R\} \subset \mathcal Z,$
where $\bm 1 \in \mathbb R^{d \times d}$ denotes the matrix with all entries equal to $1$. Under this specialization, we obtain the class of convolutional neural flow networks 
\begin{equation*}
       \left\{\left.  R \circ  F^T_{\theta_t} \circ  P : \mathbb R^{d_x \times d_x \times c} \to \mathbb R^{d_y \times d_y \times c} ~\right|~  F^T_{\theta_t} \in {}_{\mathbf n}\mathbf{F}^{\mathbf c}(\mathbb R^{d\times d}) \text{ or } {}_{\mathbf n}\mathbf{F}^{\mathbf s}(\mathbb R^{d\times d}) \right\},
\end{equation*}
whose time discretizations recover the standard classes of convolutional neural networks, including both ResNet-type and plain architectures \cite{he2016deep,bengio2017deep}.

To recover convolutional neural operators, we consider the infinite-dimensional setting with $\mathcal X$ and $\mathcal Y$ being function spaces on $\Omega$. We take $\mathcal Z = H^s(\Omega)$ as before, and restrict the admissible linear operators to convolutional integral operators. More precisely, for each $t\in(0,T]$, we assume that $\left[\mathcal W_t\right]_{i,j}$ is a convolutional integral operator of the form:
\begin{equation*}
        \left(\left[\mathcal W_t \right]_{i,j} [  z]_j\right)(x, t) = \int_{\Omega} W_{i,j}(x-y,t)[  z]_j(y,t) \mathrm d y,
\end{equation*}
so that the linear part is parameterized by $D\times D$ convolution kernels $W_{i,j}(x-y,t)$. 
These convolutional neural flow operators are closely related to the formulations introduced in~\cite{tai2024pottsmgnet,tai2025mathematical}, where similar continuous structures were studied in the context of image processing. Moreover, the semi-implicit discretization of the above convolutional neural flow operators leads to a plain convolutional architecture that is closely related to the motivation of FNO~\cite{li2021fourier}.

\subsection{Activation functions}\label{sec:activation}

The neural flow framework assumes the existence of an activation function 
$\sigma^\ell$ satisfying the implicit relation~\eqref{eq:implicit-sigma} for a 
general nonlinear map $\psi$. While stated abstractly, this assumption is not 
restrictive in practice: most commonly used activation functions satisfy it for 
a suitably chosen $\psi$. In particular, sigmoid and softmax activations satisfy 
relation~\eqref{eq:implicit-sigma}, as shown in Section~3.1 of~\cite{Jia2020} 
and Sections~5.2 and Appendix~A of~\cite{tai2024pottsmgnet}, and activations incorporating layer normalization are covered in Section~2.4 of~\cite{tai2025mathematical1}.

In the present work, we restrict attention to the parameterized Leaky ReLU 
family, which includes ReLU as the special case $a = 0$. For $a \in \mathbb{R}$, 
we define
\begin{equation*}\label{eq:LeakyReLU}
    \sigma_a(t) =
    \begin{cases}
        t,  & t \ge 0, \\
        at, & t < 0.
    \end{cases}
\end{equation*}
Within this family, the implicit relation~\eqref{eq:implicit-sigma} can be made 
explicit. Taking $\psi = \sigma_a$, the function $\sigma^\ell$ satisfying 
$\sigma^\ell\bigl(z - \Delta t\,\alpha^\ell \psi(z)\bigr) = z$ can be determined 
explicitly. In particular, when $\alpha^\ell = 1$, one obtains
\begin{equation}\label{eq:1}
    \sigma^\ell = \frac{1}{1 - \Delta t}\,\sigma_{\frac{1-\Delta t}{1-a\Delta t}},
\end{equation}
which belongs to the same Leaky ReLU family and is uniform across layers. 
Consequently, \eqref{eq:DisNFS} recovers a plain network or operator with a 
uniform activation. The derivation of~\eqref{eq:1} follows from elementary 
calculations and is omitted.

In the remainder of this work, unless otherwise stated, the activation $\sigma$ 
in the composition model~\eqref{eq:CompScheme} and the nonlinear map $\psi$ in 
the separation model~\eqref{eq:SepScheme} are both taken from the Leaky ReLU 
family, with the specific choice stated explicitly whenever needed.

\section{Well-posedness and universal approximation properties of neural flow 
models}\label{sec:flowUAP}
This section establishes two foundational theoretical properties of the proposed neural flow framework. First, we prove the well-posedness of the continuous-depth neural flow models, guaranteeing that the evolution equation~\eqref{evolution equation} admits a unique solution that depends continuously on the model parameters. This is a necessary prerequisite for the framework to be mathematically sound and for the flow map $\mathcal{F}^T_{\theta_t}$ to be well defined as an approximation tool. Second, we establish universal approximation properties for the resulting flow maps, showing that any continuous map $\mathcal{O} \colon \mathcal{C} \subset \mathcal{X} \to \mathcal{Y}$ between abstract Hilbert spaces can be approximated to arbitrary accuracy by our proposed neural flow operator. Crucially, both results are established within a single unified framework that covers composition- and separation-structured flows, finite-dimensional and infinite-dimensional settings, and Hilbert spaces that need not arise as function spaces.

\subsection{Well-posedness of neural flow models}
Well-posedness of the neural flow models is established in two steps. We first show that the evolution equation~\eqref{evolution equation} admits a unique solution for any piecewise constant parameter $\theta_t$, and then show that this solution depends locally Lipschitz continuously on the parameters, ensuring that small perturbations in the network parameters produce small changes in the resulting flow. Together, these results confirm that the neural flow models are mathematically well defined and stable with respect to parameter variation, which is essential for both the approximation theory developed below and for the practical use of gradient-based training.

For the purposes of the proofs that follow, we introduce norms on vector-valued functions and on the parameter space. For $z = (z_1, \dots, z_D)$ and parameter 
$\theta_t = (\mathcal{W}_t, \bm{b}_t, \alpha_t)$, we define
\begin{equation*}
    \| z \|_{\infty}(x,t) := \max_{1 \le i \le D} |z_i(x,t)|,
\end{equation*}
\begin{equation*}
    \| z \|_{\infty,\infty} := \max_{1 \le i \le D} 
    \sup_{t \in [0,T],\, x \in \Omega} |z_i(x,t)|,
\end{equation*}
\begin{equation*}
    \| \theta \|_{\infty,\infty} := \max\!\left\{ 
    \sup_{t \in [0,T]} \|\mathcal{W}_t\|_{\infty},\; 
    \sup_{t \in [0,T]} \|\bm{b}_t\|_{\infty},\; 
    \|\alpha_t\|_{C([0,T])} \right\}.
\end{equation*}

The following result establishes the well-posedness of the neural flow equation for piecewise constant parameters and, consequently, the well-definedness of the associated flow map.
\begin{theorem}
    For any $\theta_t$ that is piecewise constant on $(0,T]$, the evolution equation
    \begin{equation}\label{evolution equation}
        \left\{
        \begin{aligned}
            \frac{\partial z(x,t)}{\partial t} &= \Phi_{\theta_t}(z(x,t)), 
            \quad 0 < t \le T,\\
            z(x,0) &= z_0(x),
        \end{aligned}
        \right.
    \end{equation}
    admits a unique solution $z(x,t)$ for any initial condition $z_0(x)$, where
    $\Phi_{\theta_t}(z) = \Phi_{\theta_t}^c(z)$ or $\Phi_{\theta_t}^s(z)$ as 
    defined in \eqref{eq:CompScheme} and \eqref{eq:SepScheme}, with activation 
    function $\sigma_a(\cdot)$ for any $a \in \mathbb{R}$.
    Consequently, the flow map $\mathcal{F}^T_{\theta_t} \colon z(x,0) \mapsto z(x,T)$ 
    is well defined.
\end{theorem}

\begin{proof}
    Since $\theta_t$ is piecewise constant, it suffices to establish existence and uniqueness for the system with constant par ameter   $\theta$:
    \begin{equation*}
        \left\{
\begin{aligned}
\frac{\partial z(x,t)}{\partial t} &= \Phi_{\theta}(z(x,t)), \quad 0 < t \le T,\\
z(x,0) &= z_0(x).
\end{aligned}
\right.
    \end{equation*}

    \begin{itemize}
        \item 
When $\Phi_{\theta}(z(x,t))=\Phi_{\theta}^c  (z(x,t))=\sigma_{a}(\mathcal W z(x,t) + \bm b)$ for certain $\mathcal W, \bm b$, the map $\Phi_{\theta}^c$ is Lipschitz continuous with respect to $z$:
 \begin{equation*}
        \begin{aligned}
||\Phi_{\theta}^c(z^1)-\Phi_{\theta}^c(z^2)||_{\infty}&=||\sigma _{a}(\mathcal W z^1+\bm b )-\sigma _{a}(\mathcal W z^2+\bm b  )||_{\infty}\\
 &\le \max\{1,|a|\}\cdot ||\mathcal W(z^1-z^2)||_{\infty}\\
&\le \max\{1,|a|\}\cdot||\mathcal W||_\infty||z^1-z^2||_\infty
\end{aligned}
    \end{equation*}
        \item When $\Phi_{\theta}(z(x,t))=\Phi_{\theta}^s(z(x,t))=\mathcal W z(x,t) + \bm b + \alpha \sigma_{a}(z(x, t))$ for certain $\mathcal W, \bm b, \alpha$, the map $\Phi_{\theta}^s$ is Lipschitz continuous with respect to $z$:
 \begin{equation*}
        \begin{aligned}
||\Phi_{\theta}^s(z^1)-\Phi_{\theta}^s(z^2)||_{\infty}&=||\mathcal W(z^1-z^2)+\alpha(\sigma _{a}( z^1 )-\sigma _{a}( z^2 ))||_{\infty}\\
 &\le ||\mathcal W(z^1-z^2)||_{\infty}+||\alpha(\sigma _{a}( z^1 )-\sigma _{a}( z^2 ))||_{\infty}\\
&\le (||\mathcal W||_\infty +|\alpha |\cdot \max\{1,|a|\})\cdot||z^1-z^2||_\infty
\end{aligned}
    \end{equation*}
    \end{itemize}

Thus, we have $||\Phi_{\theta}(z^1)-\Phi_{\theta}(z^2)||_{\infty,\infty}\le L||z^1-z^2||_{\infty,\infty}$ for some $L>0$.

The existence and uniqueness of solutions then follow from the classical theorem for dynamical systems.
\end{proof}

The next result establishes that the solution to \eqref{evolution equation} depends locally Lipschitz continuously on the parameter path $\theta_t$.
\begin{theorem}
    For any $T > 0$, the solution $z(x,t)$ of the evolution 
    equation~\eqref{evolution equation} depends locally Lipschitz continuously 
    on the parameter $\theta_t$ and on $|a|$ in the activation function 
    $\sigma_a(\cdot)$. Specifically, for any piecewise constant $\theta^1_t$, 
    there exists a neighborhood $V$ of $\theta^1_t$ such that for any 
    $\theta^2_t \in V$, the corresponding solutions $z^1(x,t)$ and $z^2(x,t)$ 
    with the same initial condition $z_0$ satisfy
    \begin{equation*}
        \| z^1 - z^2 \|_{\infty,\infty} 
        \le M T e^{LT} \| \theta^1 - \theta^2 \|_{\infty,\infty},
    \end{equation*}
    where $L$ is the Lipschitz constant of $\Phi_{\theta}$ with respect to $z$, 
    and $M$ is the local Lipschitz constant of $\Phi_{\theta}$ with respect 
    to $\theta$.
\end{theorem}
\begin{proof}
    We first show that $\Phi_{\theta}$ is locally Lipschitz continuous with respect to $\theta$.
When $\Phi_{\theta^i}(z(x,t))=\Phi_{\theta^i}^c  (z(x,t))$, we have
    \begin{equation*}
        \begin{aligned}
||\Phi_{\theta^1}^c(z)-\Phi_{\theta^2}^c(z)||_{\infty,\infty}&=||\sigma _{a}(\mathcal W^1 z+\bm b^1)-\sigma _{a}(\mathcal W^2z+\bm b^2)||_{\infty,\infty}\\
 &\le\max\{1,|a|\}\cdot ||(\mathcal W^1-\mathcal W^2)z+\bm b^1-\bm b^2||_{\infty,\infty}\\
 &\le \max\{1,|a|\}\cdot (||(\mathcal W^1-\mathcal W^2)z||_{\infty,\infty}+||\bm b^1-\bm b^2||_{\infty,\infty})\\
 &\le\max\{1,|a|\}\cdot (||(\mathcal W^1-\mathcal W^2)||_\infty \cdot||z||_{\infty,\infty}+||\bm b^1-\bm b^2||_{\infty,\infty})\\
  &\le \max\{1,|a|\}\cdot(||z||_{\infty,\infty}+1)||\theta^1-\theta^2||_{\infty,\infty}.
\end{aligned}
    \end{equation*}
And when $\Phi_{\theta^i}(z(x,t))=\Phi_{\theta^i}^s(z(x,t))$,
    \begin{equation*}
        \begin{aligned}
||\Phi_{\theta^1}^s(z)-\Phi_{\theta^2}^s(z)||_{\infty,\infty} 
=&||(\mathcal W^1-\mathcal W^2)z+\bm b^1-\bm b^2+\alpha^1\sigma _{a}(z)-\alpha^2\sigma _{a}(z)||_{\infty,\infty}\\
 \le &||(\mathcal W^1-\mathcal W^2)z||_{\infty,\infty}+||\bm b^1-\bm b^2||_{\infty,\infty}\\
 &+||\alpha^1\sigma _{a}(z)-\alpha^2\sigma _{a}(z)||_{\infty,\infty}\\
 \le &||(\mathcal W^1-\mathcal W^2)||_\infty\cdot||z||_{\infty,\infty}+||\bm b^1-\bm b^2||_{\infty,\infty}\\
 &+||\alpha^1-\alpha^2||_{C([0,T])}\cdot\max\{1,|a|\})\cdot||z||_{\infty,\infty}\\
  \le &(||z||_{\infty,\infty}+1+\max\{1,|a|\})\cdot||z||_{\infty,\infty} )||\theta^1-\theta^2||_{\infty,\infty}
\end{aligned}
    \end{equation*}
Then we have $||\Phi_{\theta^1}(z)-\Phi_{\theta^2}(z)||_{\infty,\infty}\le M||\theta^1-\theta^2||_{\infty,\infty}$ for some $M>0$ depending on the solution $z$.
Since  $\frac{\partial z^i(x,t)}{\partial t}=\Phi_{\theta^i_t}(z^i(x,t)),~i=1,2$ with the same initial condition, we obtain
\begin{equation*}
    \begin{aligned}
|z^1_i(x,t)-z^2_i(x,t)|&=|\int_{0}^{t}[\Phi_{\theta^1_t}(z^1)- \Phi_{\theta^2_t}(z^2)]_id\tau|\\
&\le|\int_{0}^{t}[\Phi_{\theta^1_t}(z^1)- \Phi_{\theta^1_t}(z^2)]_id\tau|+|\int_{0}^{t}[\Phi_{\theta^1_t}(z^2)- \Phi_{\theta^2_t}(z^2)]_id\tau|
\end{aligned}
\end{equation*}
where $[\cdot]_i$ denotes the $i$-th component. Therefore,
\begin{equation*}
    \begin{aligned}
||z^1-z^2||_\infty(x,t)&\le\int_{0}^{t}||\Phi_{\theta^1_t}(z^1)- \Phi_{\theta^1_t}(z^2)||_{\infty}d\tau+\int_{0}^{t}||\Phi_{\theta^1_t}(z^2)- \Phi_{\theta^2_t}(z^2)||_{\infty,\infty}d\tau\\
& \le\int_{0}^{t}L||z^1-z^2||_\infty(x,\tau)d\tau+M||\theta ^1-\theta ^2||_{\infty,\infty}t
\end{aligned}
\end{equation*}
By Gronwall's inequality, we have
\begin{equation*}
    ||z^1-z^2||_{\infty,\infty}\le MT e^{LT}||\theta^1-\theta^2||_{\infty,\infty}
\end{equation*}
\end{proof}

\subsection{Universal approximation properties of our neural flow operators}
We now establish the main result of this work: an abstract universal approximation theorem for neural flow operators in infinite-dimensional spaces. To the best of our knowledge, this is the first universal approximation result for flow models acting between infinite-dimensional spaces; all existing approximation analyses are restricted to finite-dimensional settings. Furthermore, the theorem is formulated in the framework of abstract Hilbert spaces and is therefore not limited to function spaces.

\begin{theorem}\label{them:approxi}
    Let $\mathcal{X}$ and $\mathcal{Y}$ be Hilbert spaces
     and let $\mathcal{C} \subset \mathcal{X}$ be compact. Suppose that 
    $\mathcal{O} \colon \mathcal{C} \to \mathcal{Y}$ is continuous and that 
    $\mathcal{Z}$ is a function space containing the constant functions. Then, for 
    any $\varepsilon > 0$ and $a \neq 1$, there exist $D \in \mathbb{N}^+$, bounded 
    linear operators $\mathcal{L} \colon \mathcal{X} \to \mathcal{Z}^{D}$ and 
    $\mathcal{R} \colon \mathcal{Z}^{D} \to \mathcal{Y}$, and a neural flow operator 
    $\mathcal{F}^T_{\theta_t} \in {}_{\mathbf{n}}\mathbf{F}^{\mathbf{c}}(\mathcal{Z}^D)$ 
    or $\mathcal{F}^T_{\theta_t} \in {}_{\mathbf{n}}\mathbf{F}^{\mathbf{s}}(\mathcal{Z}^D)$ 
    with activation function $\sigma_a(\cdot)$, such that
    \begin{equation*}
        \sup_{\bm{v} \in \mathcal{C}} \left\| \mathcal{O}(\bm{v}) - \mathcal{R} \circ 
        \mathcal{F}^T_{\theta_t} \circ \mathcal{L}(\bm{v}) \right\|_{\mathcal{Y}} 
        \leq \varepsilon.
    \end{equation*}
\end{theorem}

Before proving the theorem, we need the following key lemma on uniform finite-dimensional approximation of a compact set in a Hilbert space.
\begin{lemma}Let $\mathcal X$ be a Hilbert space and let $\mathcal C\subset \mathcal X$ be compact.
Define the closed linear subspace
\begin{equation*}  
\mathcal X_0:=\overline{\operatorname{span}}(\mathcal C)\subset \mathcal X.
\end{equation*}  
Then $\mathcal X_0$ is separable. Hence, there exists a countable orthonormal basis
$\{e_i\}_{i=1}^\infty$ of $\mathcal X_0$ such that for every $\varepsilon>0$ there exists $k\in\mathbb N$
with
$$
\sup_{x\in\mathcal C}\left\|x-\sum_{i=1}^k (x,e_i)e_i\right\|<\varepsilon.
$$
\end{lemma}

\begin{lemma}\label{thm:UAP_general_dim}
Let $\Omega\subset\mathbb R^{d_x}$ be compact and $F:\Omega\to\mathbb R^{d_y}$ be continuous. Then, for any $D\ge \max(4d_x+2,2d_y)$, $a\neq 1$ and any $\varepsilon>0$, there exist a neural flow map $F_{\theta_t}^T\in{}_{\mathbf n}\mathbf{F}^{\mathbf c}(\mathbb R ^{D})$ or ${}_{\mathbf n}\mathbf{F}^{\mathbf s}(\mathbb R^{D})$ with activation function $\sigma_a(\cdot)$, and linear maps $P:\mathbb R^{d_x}\to\mathbb R^{D}$ and $R:\mathbb R^{D}\to\mathbb R^{d_y}$ such that
\begin{equation*}\label{eq:UAP_general_dim}
\| R\circ F_{\theta_t}^T\circ P-F\|_{L^\infty(\Omega)}\le \varepsilon. 
\end{equation*}
\end{lemma}
\begin{proof}
    The proof for this technical lemma is detailed in Appendix~\ref{sec:appendix}.
\end{proof}

With these lemmas in hand, we present the proof.

\begin{proof}
    Since $\mathcal{C}\subset\mathcal{X}$ is compact and $\mathcal{O}$ is continuous, $\mathcal O(\mathcal{C})$ is also compact in $\mathcal{Y}$. For any $\varepsilon > 0$, there exist an orthonormal basis $\{\xi_1, \dots, \xi_m\} \subset \mathcal{Y}$ and continuous functionals $f_i: \mathcal{C} \to \mathbb{R}$ for $i = 1, \dots, m$ such that
\begin{equation*}
\sup_{\bm v \in \mathcal{C}} \left\| \mathcal O(\bm v) - \sum_{i=1}^m f_i(\bm v) \xi_i \right\|_{\mathcal{Y}} \leq \frac{\varepsilon}{3}.
\end{equation*}
It therefore suffices to construct flow operators $ \mathcal F^T_{\theta_t} $ and two affine maps $\mathcal P, \mathcal L$ satisfying
\begin{equation*}
\sup_{\bm v \in \mathcal{C}} \left\| \sum_{i=1}^m f_i(\bm v) \xi_i -\mathcal R \circ \mathcal F^T_{\theta_t} \circ \mathcal P(\bm v) \right\|_{\mathcal{Y}} \leq \frac{2\varepsilon}{3}.
\end{equation*}

Since $\mathcal{X}$ is a separable Hilbert space and $\mathcal{C}$ is compact, there exists $k \in \mathbb{N}^+$ and an orthogonal set $\{\eta_j\}_{j=1}^\infty$ in $\mathcal{X}$ such that
\begin{equation*}
\sup_{v \in \mathcal{C}} \left\| f_i(\bm v)  - f_i\left( \sum_{j=1}^k \langle \bm v, \eta_j \rangle \eta_j \right) \right\|_{\mathcal{Y}} \leq \frac{\varepsilon}{3m}, \quad \forall i = 1, \dots, m.
\end{equation*}

Define the finite-dimensional function $F_i: \mathbb{R}^k \to \mathbb{R}$ by
\begin{equation*}  
F_i(x) = f_i\left( \sum_{j=1}^k x_j \eta_j \right), \quad x \in [-M, M]^k
\end{equation*}  
where $M = \sup_i \sup_{\bm v \in \mathcal{C}} |\langle v, \eta_i \rangle| < \infty$ (guaranteed by compactness of $\mathcal{C}$).

Combining the functions $F_i$ into a vector-valued map $\mathbf{F}: \mathbb{R}^{k} \to \mathbb{R}^m$ by
\begin{equation*}
\mathbf{F}(x) = (F_1(x),F_2(x),\cdots,F_m(x))^\top, 
\end{equation*}
where $x=(x_1,x_2,\cdots,x_k) \in [-M, M]^k$.

By Lemma~\ref{thm:UAP_general_dim}, there exists a neural flow network $F^T_{\theta_t}$ on $\mathbb R^D$ and linear maps $P:\mathbb R^{k}\to\mathbb R^D,R:\mathbb R^D\to \mathbb R^m$ such that

\begin{equation*}
\sup_{x\in [-M, M]^k}\sum_{i=1}^m|F_i(x) - \tilde{F}_i(x)| \leq \frac{\varepsilon}{3}.
\end{equation*}
with $ \tilde{\mathbf{F}}(x)=(\tilde{F}_1(x),\tilde{F}_2(x),\cdots,\tilde{F}_m(x))=R\circ F^T_{\theta_t}\circ P(x)$

Next, we define the required operators.
First, let
\begin{equation*}
    L(\bm v)=\begin{pmatrix}
\langle\bm v,\eta_1 \rangle \mathbf{1}(x)\\
\vdots  \\
\langle\bm v,\eta_k \rangle\mathbf{1}(x)\\
\end{pmatrix}\in \mathcal Z^{k},
Q(v_1,\cdots,v_m)=\sum_{i=1}^m v_i\xi_i\in\mathcal Y
\end{equation*}
where $\mathbf{1}(x) \in \mathcal Z$ is the constant function with value $1$ everywhere on $\Omega$.

The dynamical system generating the flow map $F^T_{\theta_t}$ is given by
\begin{equation*}
    \frac{dz}{dt}=\mathcal W_tz+\bm b_t+\alpha_t\sigma_{a}(z),~t\in(0,T]
\end{equation*}
or
\begin{equation*}
    \frac{dz}{dt}=\sigma_a(\mathcal W_tz+\bm b_t),~t\in(0,T].
\end{equation*}

Based on this, we define the PDE system as follows:
\begin{equation*}
    \frac{\partial z}{\partial t}=\mathcal W_t z+\bm b_t+\alpha_t\sigma_{a}( z),~~t\in(0,T]
\end{equation*}
or
\begin{equation*}
     \frac{\partial z}{\partial t}=\sigma_a(\mathcal W_tz+\bm b_t),~t\in(0,T].
\end{equation*}
where $z\in \mathcal Z^{D},\mathcal  W_t\in \mathbb R^{D\times D}$.

The corresponding neural flow operator is denoted by $\mathcal F^T_{\theta_t}$. Finally, we define the operators $\mathcal P=P\circ L$ and $\mathcal R=Q\circ R$.
Then this yields
\begin{equation*}
\begin{aligned}
\mathcal{R}\circ\mathcal F^T_{\theta_t}\circ \mathcal P(\bm v) &= Q\circ R\circ  F^T_{\theta_t}\circ P \circ L(\bm v) 
=Q\circ\tilde{\mathbf F}\circ L(\bm v) \\ &
= Q\begin{pmatrix}
\tilde F_1(L(\bm v))\\
\vdots  \\ 
\tilde F_m(L(\bm v))
\end{pmatrix}  
=\sum^m_{i=1}\tilde{F}_i(L(\bm v) )\xi_i
\end{aligned}
\end{equation*}

Denote $\mathbf f(\bm v)=(f_1(\bm v),f_2(\bm v),\cdots,f_m(\bm v))^\top$. Then
\begin{align*}
    &\sup_{\bm v \in \mathcal{C}} \left\| \mathcal O(\bm v) - \mathcal{R}\circ\mathcal F^T_{\theta_t}\circ \mathcal P(\bm v) \right\|_{\mathcal{Y}}\\
    &\le \sup_{\bm v \in \mathcal{C}} \left\| \mathcal O(\bm v) - \sum_{i=1}^m f_i(\bm v) \xi_i \right\|_{\mathcal{Y}}+\sup_{\bm v \in \mathcal{C}} \left\|  \sum_{i=1}^m f_i(\bm v) \xi_i -\mathcal{R}\circ\mathcal F^T_{\theta_t}\circ \mathcal P(\bm v)\right\|_{\mathcal{Y}}\\
    &= \sup_{\bm v \in \mathcal{C}} \left\| \mathcal O(\bm v) - \sum_{i=1}^m f_i(\bm v) \xi_i  \right\|_{\mathcal{Y}}+\sup_{\bm v \in \mathcal{C}} \left\|  Q(\mathbf {f})(\bm v) -Q(\tilde{\mathbf F}\circ L )(\bm v)\right\|_{\mathcal{Y}}\\
    &\le\frac{\varepsilon}{3}+\sup_{\bm v \in \mathcal{C}}\sum^m_{i=1} \left\|  (\mathbf f -\tilde{\mathbf F}\circ L)_i(\bm{v} )\xi_i\right\|_{\mathcal{Y}}\\
    &=\frac{\varepsilon}{3}+\sup_{\bm v \in \mathcal{C}} \sum^m_{i=1}|  (\mathbf f -\tilde{\mathbf F}\circ L)_i(\bm{v} )|\left\| \xi_i\right\|_{\mathcal{Y}}\\
    &\le\frac{\varepsilon}{3}+\sup_{\bm v \in \mathcal{C}} \sum^m_{i=1}|  (\mathbf f -\mathbf F\circ L)_i(\bm{v} )|+\sup_{\bm v \in \mathcal{C}} \sum^m_{i=1}|  (\mathbf F\circ L -\tilde{ \mathbf F}\circ L)_i(\bm v )|\\
    &=\frac{\varepsilon}{3}+\sum^m_{i=1}\sup_{\bm v \in \mathcal{C}} |  f_i(\bm v) -f_i\left( \sum_{j=1}^k \langle v, \eta_j \rangle \eta_j \right)|+\frac{\varepsilon}{3}\\
    &\le\frac{\varepsilon}{3}+m\frac{\varepsilon}{3m}+\frac{\varepsilon}{3}\\
    &=\varepsilon
\end{align*}

\end{proof}

\section{Application I: Universal approximation properties for convolutional neural flow operators}\label{sec:application1}
The preceding sections establish universal approximation theorems for neural flow operators with fully connected linear layers. Convolutional architectures are of independent interest, since they incorporate translation equivariance and parameter sharing, and are therefore well suited to spatially structured problems such as partial differential equations and image-to-image mappings. We now show that the same universal approximation property also holds in this setting. Specifically, we introduce convolutional counterparts of the separation-structured and composition-structured neural flow operators, and establish that Theorem~\ref{them:approxi} remains valid when the admissible linear operators are restricted to convolutional form.

We consider the following two convolutional neural flow operators with the separation structure:
\begin{equation}\label{eq:ds_conv}
    \begin{cases}
       \frac{\partial z(x,t)}{\partial t} = \mathcal W_t*z+\bm b_t+\alpha_t\sigma_{a}(z), \qquad  t \in (0, T],\ x\in\Omega,\\
        z(x,0) = z_0(x)=(z_1^0,\cdots,z_D^0) \in \mathcal{Z}^D,
    \end{cases}
\end{equation}
and the composition structure:
\begin{equation}\label{eq:ds_conv_peter}
    \begin{cases}
       \frac{\partial z(x,t)}{\partial t} = \sigma_a(\mathcal W_t*z+\bm b_t), \qquad  t \in (0, T],\ x\in\Omega,\\
        z(x,0) = z_0(x)=(z_1^0,\cdots,z_D^0) \in \mathcal{Z}^D.
    \end{cases}
\end{equation}
Here, the convolution operator $\mathcal W_t*$ is defined componentwise by
\begin{equation*}
    [\mathcal W_t*z]_i=\sum_{j=1}^D \left[\mathcal W_t \right]_{i,j}[z]_j,\quad
    \mathcal W_t=\begin{pmatrix}
 \left[\mathcal W_t \right]_{1,1} & \left[\mathcal W_t \right]_{1,2} &\cdots  & \left[\mathcal W_t \right]_{1,D}\\
 \left[\mathcal W_t \right]_{2,1} & \left[\mathcal W_t \right]_{2,2} & \cdots & \left[\mathcal W_t \right]_{2,D}\\
 \vdots  & \vdots & \ddots  &\vdots\\
 \left[\mathcal W_t \right]_{D,1} & \left[\mathcal W_t \right]_{D,2} &\cdots  & \left[\mathcal W_t \right]_{D,D}
\end{pmatrix},
\end{equation*}
for $z\in\mathcal Z ^D$ where each block operator $\left[\mathcal W_t \right]_{i,j}$ is a convolution operator on $\mathcal Z$ of the form
\begin{equation*}
        \left(\left[\mathcal W_t \right]_{i,j} [z]_j\right)(x, t) = \int_\Omega W_{i,j}(x-y,t)[z]_j(y,t)\,\mathrm d y.
\end{equation*}

The following result shows that the abstract universal approximation theorem remains valid when the admissible linear operators are restricted to convolutional form.

\begin{theorem}\label{them:approxi_conv}
Let $\mathcal X$ and $\mathcal Y$ be two Hilbert spaces, let $\mathcal C \subset \mathcal X$ be compact, and let $\mathcal O: \mathcal C \subset \mathcal X \to \mathcal Y$ be continuous. Assume that $\mathcal Z$ is a function space on $\Omega$ containing the constant functions. Then, for any $\varepsilon > 0$ and $a \neq 1$, there exist $D \in \mathbb N^+$, two bounded linear operators
$\mathcal P: \mathcal X \to \mathcal Z^{ D}$ and $\mathcal R: \mathcal Z^{D} \to \mathcal Y$, and a convolutional neural flow operator $\mathcal F^T_{\theta_t}$ generated by either \eqref{eq:ds_conv} or \eqref{eq:ds_conv_peter} with activation function $\sigma_a(\cdot)$, such that
\begin{equation*}  
    \sup_{\bm v \in \mathcal C}\left\| \mathcal O(\bm v) - \mathcal R \circ \mathcal F^T_{\theta_t}  \circ \mathcal P (\bm v)\right\|_{\mathcal Y} \le \varepsilon.
\end{equation*}
\end{theorem}

\begin{proof}
The proof follows the same strategy as that of Theorem~\ref{them:approxi}. It suffices to show that the fully connected linear dynamics used there can be realized exactly by convolutional operators with a suitable choice of kernels.

Recall from the proof of Theorem~\ref{them:approxi} that the linear operator $\widetilde{\mathcal W}_t$ is a piecewise constant matrix-valued function of time. To reproduce its action by convolution, we choose the kernel functions
\begin{equation*}
    W_{i,j}(x-y,t)=\frac{[\widetilde{\mathcal W}_t]_{i,j}}{|\Omega|}.
\end{equation*}
Then, for any vector-valued function of the form $
    z=(c_1\mathbf{1}(x),c_2\mathbf{1}(x),\cdots,c_D\mathbf{1}(x))$,
we have
\begin{equation*}
\begin{aligned}
    [\mathcal W_t*z]_i
    &=\sum_{j=1}^D \left[\mathcal W_t \right]_{i,j}[z]_j =\sum_{j=1}^D \int_\Omega W_{i,j}(x-y,t)[z]_j(y,t)\,\mathrm d y\\
    &=\sum_{j=1}^D \int_\Omega \frac{[\widetilde{\mathcal W}_t]_{i,j}}{|\Omega|} c_j\mathbf{1}(y)\,\mathrm d y =\sum_{j=1}^D [\widetilde{\mathcal W}_t]_{i,j}c_j\\
    &=\sum_{j=1}^D [\widetilde{\mathcal W}_t]_{i,j}[z]_j(x).
\end{aligned}
\end{equation*}
Hence, we have $\mathcal W_t*z=\widetilde{\mathcal W}_t z$.

Therefore, on the class of constant-in-space latent states used in the proof of Theorem~\ref{them:approxi}, the convolutional operator $\mathcal W_t*$ reproduces exactly the same action as the fully connected matrix $\widetilde{\mathcal W}_t$. Consequently, the convolutional dynamical systems \eqref{eq:ds_conv} and \eqref{eq:ds_conv_peter} generate the same flow maps as their fully connected counterparts on this invariant class of states. The universal approximation property established in Theorem~\ref{them:approxi} therefore carries over directly to the convolutional setting.
\end{proof}

As a particular case, this theorem also yields the universal approximation property for convolutional neural flow networks by taking $\mathcal X = \mathbb R^{d_x\times d_x \times c}$ and $\mathcal Y = \mathbb R^{k}$ or $\mathbb R^{d_y \times d_y \times c}$ to be finite-dimensional Hilbert spaces.

\section{Application II: Universal approximation properties of neural operators with finite depth}\label{sec:application2}
The universal approximation results established in the preceding sections are formulated for neural flow operators in continuous time, governed by the evolution equation~\eqref{evolution equation}. In practice, however, neural networks are finite-depth objects, obtained by discretizing the continuous flow in time. This section bridges that gap: we leverage numerical time discretization schemes to convert the continuous-time universal approximation results into 
corresponding approximation properties for neural operators of finite depth, thereby providing theoretical guarantees directly at the level of practical architectures.

Specifically, applying a time discretization to the evolution 
equation~\eqref{evolution equation} yields a composition of finitely many residual-type update steps, each corresponding to one layer of a neural network. The depth of the resulting network is thus determined by the number of discretization steps. We consider several settings to demonstrate the breadth of the approach: fully connected versus convolutional architectures, and residual versus nonresidual (plain) architectures. In each case, we show that the 
corresponding finite-depth neural operators inherit the universal approximation property from their continuous-time counterparts, with the approximation error controlled jointly by the expressiveness of the network and the accuracy of the discretization scheme.

\begin{theorem}\label{UAP_fixed_para}
Let $\mathcal X$ and $\mathcal Y$ be two Hilbert spaces and let $\mathcal C \subset \mathcal X$ be compact. Assume that $\mathcal O: \mathcal C \subset \mathcal X \to \mathcal Y$ is continuous and that $\mathcal Z$ is a function space containing the constant functions. For any $\varepsilon>0$, there exist a finite-depth neural operator $\widetilde{\mathcal O}$ with activation function $\sigma_\gamma(\cdot)$, defined by
    \begin{equation*}
\begin{cases}
        & z^0 = \mathcal P \bm v \in \mathcal Z^{D}, \\
        & z^{\ell} = \sigma_\gamma \left( {\mathcal W}^\ell z^{\ell-1} + {\bm b}^\ell \right) \in \mathcal Z^{D}, \quad \ell = 1:L, \\
        & \widetilde{\mathcal O} (\bm v) = \bm u = \mathcal R z^L \in \mathcal Y,
    \end{cases}
\end{equation*}
such that
    \begin{equation*}
        \sup_{\bm v \in \mathcal C}\|\mathcal O(\bm v)-\widetilde{\mathcal O}(\bm v)\|_{\mathcal Y}\le \varepsilon.
    \end{equation*}
\end{theorem}

The proof relies on two auxiliary lemmas concerning the numerical discretization of the underlying continuous flow. We state and prove them first.
\begin{lemma}\label{time correction}
    For any neural flow operator $\mathcal F^T_{\theta_t}$ defined on a compact set $\mathcal D\subset \mathcal Z^D$ and $\varepsilon>0$, there exist $\delta>0$ such that   for any $\Delta t <\delta$, one can construct a time‑corrected neural flow operator $\mathcal G^{T'}_{\theta'_t}$ satisfies:
    \begin{enumerate}
        \item $\sup_{  z \in \mathcal D}||\mathcal F^T_{\theta_t}(z)-\mathcal G^{T'}_{\theta'_t}(z)||_{\mathcal Z^D}<\varepsilon$
        \item $\theta'_t$ is obtained from $\theta_t$ solely by adjusting time intervals, so that every continuous segment of $\theta'_t$ has length equal to an integer multiple of $\Delta t$.
    \end{enumerate}
\end{lemma}
\begin{proof}
    It suffices to prove the statement for the finite‑dimensional realization (neural flow network). The construction relies on two continuity properties of the flow: uniform continuity in time and continuous dependence on the initial value. The former is guaranteed by restricting the total time \(T\) on each continuous parameter segment to a bounded interval (e.g., \(0<T<1\)); the latter follows from the boundedness of \(\Omega\) and the continuity of the flow map.

    Consider a Neural Flow Networks $F^T_{\theta_t}\in{}_{\mathbf n}\mathbf{F}(\mathbb R^D)$, 
\begin{equation*}
    F^T_{\theta_t} (x) = f_{\theta_n}^{\tau_n} \circ \cdots \circ f_{\theta_2}^{\tau_2} \circ f_{\theta_1}^{\tau_1} (x), \quad T = \sum_{i=1}^{n} \tau_i.
\end{equation*}
where each $\theta_i$ is constant on a time interval of length \(\tau_i\). 

For each component flow $f_{\theta_i}^{\tau_i}$ and any $\varepsilon>0$, there exist a $\delta_i>0$ such that whenever $|\tau_i-\tau_i'|<\delta_i$ and $||x'-x''||<\delta_i$, we have $||f^{\tau_i}_{\theta_i}(x')-f^{\tau_i'}_{\theta_i}(x'')||<\varepsilon $.

Consequently, for any \(\varepsilon>0\) and $x\in\mathbb R^D$, we can choose a \(\delta\) such that if \(|\tau_i - \tau_i'| < \delta\) for all \(i=1,\dots,n\), then
    \[
        \| F^T_{\theta_t}(x) - G^{T'}_{\theta'_t}(x) \| < \varepsilon,
    \]
where $ G^{T'}_{\theta'_t} (x) = f_{\theta_n}^{\tau'_n} \circ \cdots \circ f_{\theta_2}^{\tau'_2} \circ f_{\theta_1}^{\tau'_1} (x)$, and $T' = \sum_{i=1}^{n} \tau_i'.$
Hence, by taking \(\Delta t < \delta\) and adjusting each \(\tau_i\) to a nearby integer multiple of \(\Delta t\), we obtain a time‑corrected flow that approximates the original one within error \(\varepsilon\).
\end{proof}

Based on this lemma, we adopt the following semi-implicit scheme:
\begin{equation}\label{eq:main_iteration}
    \begin{cases}
        Z_{k+\frac{1}{2}}=Z_k+\Delta t\mathcal{W}_{k}Z_k+\Delta t\bm b_k\\
        Z_{k+1}=Z_{k+\frac{1}{2}}+\Delta t\alpha_{k+1}\sigma_{a}(Z_{k+1})\\
        Z_0=Z_0(x)
    \end{cases}
\end{equation}
where $t_k=k\Delta t,\mathcal{W}_{k}:=\mathcal{W}_{t_k},\bm b_k=\bm b_{t_k}$.
Let \(\mathcal O_{\theta_t}^{\Delta t}\) denote the operator defined by this numerical scheme, i.e., \(\mathcal O_{\theta_t}^{\Delta t}(Z_0) = Z_L\), where \(Z_L\) is the solution after \(L\) steps. Then we have the following error estimation lemma:

\begin{lemma}\label{error estimation}
For any $\varepsilon>0$, the time-corrected numerical scheme $\mathcal O_{\theta_t}^{\Delta t}$ can approximate the neural flow operator $\mathcal F^T_{\theta_t}$ with the following accuracy
\begin{equation*}
    \sup_{  z \in \mathcal D}||\mathcal F^T_{\theta_t}(z)-\mathcal O_{\theta_t}^{\Delta t}(z)||_{\mathcal Z^D}\le C_1\Delta t+\varepsilon
\end{equation*}
where the constant $C_1$ depends on $\varepsilon$ and is independent of $\Delta t$.
\end{lemma}

 The proof follows from Lemma~\ref{time correction} and standard error estimates for the semi‑ implicit discretization of ordinary differential equations. The time correction ensures that the parameter path is aligned with the time grid, while the classical numerical analysis provides the $O(\Delta t)$ local truncation error, which accumulates to a global error of order $C_1 \Delta t$.

We are now ready to prove Theorem~\ref{UAP_fixed_para}.

\begin{proof}{\bf (Proof of Theorem~\ref{UAP_fixed_para})}
The universal approximation property for the continuous neural flow operator $\mathcal F^T_{\theta_t}$ has been established in Section \ref{sec:flowUAP}. For any $\varepsilon>0$, there exist two bounded linear operator $\mathcal P$ and $\mathcal R$ such that
\begin{equation*}
    \sup_{\bm v \in \mathcal C}\left\| \mathcal O(\bm v) - \mathcal R \circ \mathcal F^T_{\theta_t} \circ \mathcal P (\bm v)\right\|_{\mathcal Y} \le \frac{\varepsilon}{4}.
\end{equation*}

    Consider $\mathcal F^T_{\theta_t}$ being defined by the dynamical system
    \begin{equation*}
        \frac{\partial z(x,t)}{\partial t} = \mathcal W_t z + \bm b_t + \alpha_t \sigma_{a}(z), \qquad x\in\Omega,~t\in(0,T],
    \end{equation*}
    where, following the construction in \cite{cai2025achieving}, the coefficient $\alpha_t$ may be restricted to the set $\{0,1\}$. A direct discretization faces a technical difficulty: the parameter path $\theta_t = (\mathcal W_t, \bm b_t, \alpha_t)$ is only piecewise continuous, and the lengths of the continuous segments may be incommensurate, preventing the use of a uniform time step $\Delta t$. Lemma~\ref{time correction} resolves this by allowing us to slightly perturb the segment lengths so that they all become integer multiples of a sufficiently small $\Delta t$, while controlling the approximation error.

    Applying Lemma~\ref{time correction} with error tolerance $\frac{\varepsilon}{4||\mathcal R||}$ yields a time‑corrected operator $\mathcal G^{T'}_{\theta'_t}$ and a $\delta>0$. Choose $\Delta t < \min(\delta, \frac{\varepsilon}{4C_1||\mathcal R||})$, where $C_1$ is the constant from Lemma~\ref{error estimation}. Discretizing $\mathcal G^{T'}_{\theta'_t}$ with the semi‑implicit scheme~\eqref{eq:main_iteration} gives an operator $\mathcal O_{\theta'_t}^{\Delta t}$. By Lemma~\ref{error estimation}, we have
    \[
    \sup_{  z \in \mathcal P(\mathcal C)}\| \mathcal G^{T'}_{\theta'_t}(z) - \mathcal O_{\theta'_t}^{\Delta t}(z) \|_{\mathcal Z^D} \le C_1 \Delta t + \frac{\varepsilon}{4||\mathcal R||} < \frac{\varepsilon}{2||\mathcal R||}.
      \]
    Combining this with the bound from Lemma~\ref{time correction} yields
    \[
    \begin{aligned}
    &\sup_{  z \in \mathcal P(\mathcal C)}\| \mathcal F^T_{\theta_t}(z) - \mathcal O_{\theta'_t}^{\Delta t} (z)\|_{\mathcal Z^D} \\ &\le  \sup_{  z \in \mathcal P(\mathcal C)}\| \mathcal F^T_{\theta_t}(z) - \mathcal G^{T'}_{\theta'_t}(z) \|_{\mathcal Z^D}  +  \sup_{  z \in \mathcal P(\mathcal C)}\| \mathcal G^{T'}_{\theta'_t}(z) - \mathcal O_{\theta'_t}^{\Delta t}(z) \|_{\mathcal Z^D}  \\
    &\le \frac{\varepsilon}{4||\mathcal R||} + \frac{\varepsilon}{2||\mathcal R||} = \frac{3\varepsilon}{4||\mathcal R||}.
    \end{aligned}
    \]
    Thus, we have
    \begin{equation*}
    \begin{aligned}
        &\sup_{\bm v \in \mathcal C}\|\mathcal O(\bm v) - \mathcal R\circ \mathcal O_{\theta'_t}^{\Delta t}\circ\mathcal P(\bm v)\|_{\mathcal{Y}} \\
        \le &\sup_{\bm v \in \mathcal C}\|\mathcal O(\bm v) - \mathcal R \circ \mathcal F^T_{\theta_t} \circ \mathcal P(\bm v)\|_{\mathcal{Y}}+\sup_{\bm v \in \mathcal C}\|\mathcal R \circ \mathcal F^T_{\theta_t} \circ \mathcal P (\bm v)- \mathcal R \circ \mathcal O_{\theta'_t}^{\Delta t} \circ \mathcal P(\bm v)\|_{\mathcal{Y}}\\
        \le &\frac{\varepsilon}{4}+||\mathcal R||\cdot\sup_{z \in \mathcal P(\mathcal C)}\| \mathcal F^T_{\theta_t}(z)  -  \mathcal O_{\theta'_t}^{\Delta t}(z) \|_{\mathcal{Z}^D}
        \le \frac{\varepsilon}{4}+||\mathcal R||\cdot\frac{3\varepsilon}{4||\mathcal R||} = \varepsilon.
    \end{aligned}
    \end{equation*}

    Finally, we rewrite the iteration~\eqref{eq:main_iteration} in a layer‑wise form. When $\alpha_k=0$, the update reduces to an affine transformation; consecutive affine layers can be merged. After merging, we may assume that every nonlinear step has $\alpha_k = 1$. For the Leaky-ReLU activation, the second equation in~\eqref{eq:main_iteration} admits the closed‑form solution
    \begin{equation*}
    \begin{aligned}
         Z_{k+1} &=
         \begin{cases}
             \dfrac{1}{1-\Delta t} \, Z_{k+\frac{1}{2}}, & Z_{k+\frac{1}{2}} \ge 0,\\[8pt]
             \dfrac{1}{1-\Delta t a} \, Z_{k+\frac{1}{2}}, & Z_{k+\frac{1}{2}} < 0,
         \end{cases} \\
         &= \sigma_{\frac{1-\Delta t}{1-\Delta t a}}\!\!\left( \frac{Z_{k+\frac{1}{2}}}{1-\Delta t} \right) \\
         &= \sigma_{\gamma}\!\left( \widehat{\mathcal W}_{k} Z_k + \widehat{\bm b}_k \right),
    \end{aligned}
    \end{equation*}
    where $\gamma = \frac{1-\Delta t}{1-\Delta t a}$, and $\widehat{\mathcal W}_{k}, \widehat{\bm b}_k$ are suitable affine coefficients derived from $\mathcal W_k, \bm b_k$ and the scaling factor $1/(1-\Delta t)$. Stacking these updates and adding the initial lifting map $\mathcal P$ and the final projection map $\mathcal R$ exactly reproduces the fixed‑depth architecture
    \begin{equation*}
        \begin{cases}
              z^0 = \mathcal P \bm v \in \mathcal Z^{D}, \\[4pt]
              z^{\ell} = \sigma_\gamma \left( {\mathcal W}^\ell   z^{\ell-1} +  {\bm b}^\ell \right) \in \mathcal Z^{D}, \quad \ell = 1, \dots, L, \\[4pt]
            \widetilde{\mathcal O} (\bm v) = \bm u = \mathcal R   z^L \in \mathcal Y.
        \end{cases}
    \end{equation*}
    The operator $\widetilde{\mathcal O}$ constructed this way coincides with $\mathcal R\circ\mathcal O_{\theta'_t}^{\Delta t}\circ\mathcal P$. This completes the proof.
\end{proof}

A similar universal approximation result for the ResNet architecture can be obtained by applying a forward Euler discretization to the composition model.

\begin{theorem} \label{UAP_resnet}
Let $\mathcal X$ and $\mathcal Y$ be two Hilbert spaces, not necessarily function spaces, and let $\mathcal C \subset \mathcal X$ be compact. Assume that $\mathcal O: \mathcal C \subset \mathcal X \to \mathcal Y$ is continuous and that $\mathcal Z$ is a function space containing the constant functions. For any $\varepsilon>0$ and $0\le a<1$, there exist a neural operator $\tilde {\mathcal O}$ defined by
    \begin{equation*}
\begin{cases}
&  z^0 = \mathcal P \bm v \in \mathcal Z^D,\\
&  z^\ell =   z^{\ell-1} + \sigma_a\bigl(\mathcal W^\ell   z^{\ell-1} + \bm b^\ell\bigr)
\in \mathcal Z^D,\quad \ell=1:L,\\
&\widetilde{\mathcal O} (\bm v) = \bm u = \mathcal R   z^L \in \mathcal Y,
\end{cases}
\end{equation*}
    such that
    \begin{equation*}
        \sup_{\bm v \in \mathcal C}||\mathcal O(\bm v)-\tilde {\mathcal O}(\bm v)||_{\mathcal Y}\le \varepsilon.
    \end{equation*}
\end{theorem} 

\begin{proof}
    This theorem can be proven using a similar approach. 
\end{proof}
\par

Building on the discussion in Section~\ref{sec:application1}, the universal approximation results established above can be directly extended to convolutional architectures. The key is that, as shown in Section~\ref{sec:application1}, a standard fully-connected layer can be exactly realized by a convolutional layer with a specially designed constant kernel. Consequently, the same approximation guarantees hold for their convolutional counterparts.
\begin{corollary}
Let $\mathcal X$ and $\mathcal Y$ be two Hilbert spaces and let $\mathcal C \subset \mathcal X$ be compact. Assume that $\mathcal O: \mathcal C  \subset \mathcal X \to \mathcal Y$ is continuous and that $\mathcal Z$ is a function space containing the constant functions. For any $\varepsilon>0$, there exist a convolution neural operator $\tilde {\mathcal O}$ defined by
    \begin{equation*}
\begin{cases}
        &  z^0 = \mathcal P \bm v \in \mathcal Z^{D}, \\
        &  z^{\ell} = \sigma_\gamma \left( {\mathcal W}^\ell *  z^{\ell-1} +  {\bm b}^\ell \right) \in \mathcal Z^{D}, \quad \ell = 1:L, \\
        &\widetilde{\mathcal O} (\bm v) = \bm u = \mathcal R   z^L \in \mathcal Y,
    \end{cases}
\end{equation*}
    such that
    \begin{equation*}
        \sup_{\bm v \in \mathcal C}||\mathcal O(\bm v)-\tilde {\mathcal O}(\bm v)||_{\mathcal Y}\le \varepsilon.
    \end{equation*}
\end{corollary}

\begin{corollary}
Let $\mathcal X$ and $\mathcal Y$ be two Hilbert spaces and let $\mathcal C \subset \mathcal X$ be compact. Assume that $\mathcal O: \mathcal C \subset \mathcal X \to \mathcal Y$ is continuous and that $\mathcal Z$ is a function space containing the constant functions. For any $\varepsilon>0$ and $0\le a<1$, there exist a convolution neural operator $\tilde {\mathcal O}$ defined by
    \begin{equation*}
\begin{cases}
&  z^0 = \mathcal P \bm v \in \mathcal Z^D,\\
&  z^\ell =   z^{\ell-1} + \sigma_a\bigl(\mathcal W^\ell *    z^{\ell-1} + \bm b^\ell\bigr)
\in \mathcal Z^D,\quad \ell=1:L,\\
&\widetilde{\mathcal O} (\bm v) = \bm u = \mathcal R   z^L \in \mathcal Y,
\end{cases}
\end{equation*}
    such that
    \begin{equation*}
        \sup_{\bm v \in \mathcal C}||\mathcal O(\bm v)-\tilde {\mathcal O}(\bm v)||_{\mathcal Y}\le \varepsilon.
    \end{equation*}
\end{corollary}

We first observe that the above results apply directly to deep neural networks by taking $\mathcal X$ and $\mathcal Y$ to be Euclidean spaces, such as $\mathbb R^d$, with $\mathcal Z=\mathbb R$, and to deep convolutional neural networks by taking $\mathcal X$ and $\mathcal Y$ to be image spaces such as $\mathbb R^{d\times d\times c}$ with $\mathcal Z=\mathbb R^{d\times d}$. In both cases, the theory covers plain as well as ResNet-type architectures.

In summary, this section establishes a unified framework for the universal approximation properties of finite-depth fully connected and convolutional neural networks and neural operators, under both plain and ResNet-type architectures. Existing results in the literature are often derived separately for particular model classes, whereas the present framework treats them in a common manner.

\section{Conclusions}\label{sec:conclusions}
In this work, we introduced an abstract neural flow framework that provides a unified continuous-depth formulation for both neural networks and neural operators. A key feature of the framework is that it is formulated at the level of abstract Hilbert spaces, and is therefore not restricted to function spaces: the results apply equally to any pair of Hilbert spaces, whether or not they arise as spaces of functions on a domain. This generality allows finite-dimensional function approximation and infinite-dimensional operator approximation to be treated within a common mathematical setting, and enables both composition- and separation-based flow models to be analyzed in a unified manner.

Within this framework, we established well-posedness for the associated neural flow models and proved universal approximation properties for the corresponding flow maps. To the best of our knowledge, this is the first universal approximation result for flow-based models acting between infinite-dimensional spaces; all 
existing results of this type are restricted to finite-dimensional settings. By further restricting the admissible linear transformations to convolutional operators, we also obtained universal approximation results for convolutional neural flow models. Moreover, through suitable time discretizations, the same framework yields approximation results for finite-depth architectures, including both fully connected and convolutional neural networks and neural operators, under both plain and ResNet-type structures. From the perspective of numerical analysis, this framework offers a principled way to relate continuous neural flows to their discrete realizations, and thus provides a unified interpretation of several classes of learning models that are often studied separately.

Several directions remain for future study. The present analysis is carried out for the parameterized Leaky ReLU family, and it would be of interest to extend the framework to more general activation functions. It is also natural to seek quantitative versions of the approximation theory, including explicit error bounds and approximation rates. Since finite-depth architectures arise through time discretization of the underlying flow, ideas from the numerical analysis of ODEs and PDEs may provide useful guidance for designing architectures with improved interpretability, stability, or efficiency. Another natural direction is whether the present framework can be extended to broader model classes, such as Transformer-type architectures, along the lines studied in~\cite{tai2025mathematical1,tai2024pottsmgnet}.

\bibliographystyle{siamplain}
\bibliography{ref}

\begin{thebibliography}{10}

\bibitem{alvarezlopez2024interplay}
{\sc A.~{\'A}lvarez-L{\'o}pez, A.~Hadj~Slimane, and E.~Zuazua}, {\em Interplay between depth and width for interpolation in neural {ODE}s}, Neural Networks, 180 (2024), p.~106640, \url{https://doi.org/10.1016/j.neunet.2024.106640}, \url{https://doi.org/10.1016/j.neunet.2024.106640}.

\bibitem{bach2017breaking}
{\sc F.~Bach}, {\em Breaking the curse of dimensionality with convex neural networks}, Journal of Machine Learning Research, 18 (2017), pp.~1--53, \url{http://jmlr.org/papers/v18/14-546.html}.

\bibitem{bao2023approximation}
{\sc C.~Bao, Q.~Li, Z.~Shen, C.~Tai, L.~Wu, and X.~Xiang}, {\em Approximation analysis of convolutional neural networks}, East Asian Journal on Applied Mathematics, 13 (2023), pp.~524--549.

\bibitem{barron1993universal}
{\sc A.~R. Barron}, {\em Universal approximation bounds for superpositions of a sigmoidal function}, IEEE Transactions on Information Theory, 39 (1993), pp.~930--945.

\bibitem{bartlett1998almost}
{\sc P.~Bartlett, V.~Maiorov, and R.~Meir}, {\em Almost linear vc dimension bounds for piecewise polynomial networks}, Advances in neural information processing systems, 11 (1998).

\bibitem{bartlett2019nearly}
{\sc P.~L. Bartlett, N.~Harvey, C.~Liaw, and A.~Mehrabian}, {\em Nearly-tight vc-dimension and pseudodimension bounds for piecewise linear neural networks}, The Journal of Machine Learning Research, 20 (2019), pp.~2285--2301.

\bibitem{bengio2017deep}
{\sc Y.~Bengio, I.~Goodfellow, and A.~Courville}, {\em Deep Learning}, MIT Press, Cambridge, MA, 2017.

\bibitem{cai2023achieve}
{\sc Y.~Cai}, {\em Achieve the minimum width of neural networks for universal approximation}, in The Eleventh International Conference on Learning Representations, 2023, \url{https://openreview.net/forum?id=hfUJ4ShyDEU}.

\bibitem{cai2025achieving}
{\sc Y.~Cai and Y.~Duan}, {\em Achieving universal approximation and universal interpolation via nonlinearity of control families}, arXiv,  (2025), \url{https://doi.org/https://arxiv.org/abs/2510.03676}.

\bibitem{cao2021choose}
{\sc S.~Cao}, {\em Choose a transformer: Fourier or galerkin}, Advances in neural information processing systems, 34 (2021), pp.~24924--24940.

\bibitem{chang2018multilevel}
{\sc B.~Chang, L.~Meng, E.~Haber, F.~Tung, and D.~Begert}, {\em Multi-level residual networks from dynamical systems view}, in International Conference on Learning Representations, 2018, \url{https://openreview.net/forum?id=SyJS-OgR-}.

\bibitem{chen2018neuralode}
{\sc R.~T.~Q. Chen, Y.~Rubanova, J.~Bettencourt, and D.~Duvenaud}, {\em Neural ordinary differential equations}, in Advances in Neural Information Processing Systems, vol.~31, 2018.

\bibitem{chen1995universal}
{\sc T.~Chen and H.~Chen}, {\em Universal approximation to nonlinear operators by neural networks with arbitrary activation functions and its application to dynamical systems}, IEEE Transactions on Neural Networks, 6 (1995), pp.~911--917, \url{https://doi.org/10.1109/72.392253}.

\bibitem{cheng2025interpolation}
{\sc J.~Cheng, Q.~Li, T.~Lin, and Z.~Shen}, {\em Interpolation, approximation, and controllability of deep neural networks}, SIAM Journal on Control and Optimization, 63 (2025), pp.~625--649, \url{https://doi.org/10.1137/23M1599744}.

\bibitem{cuchiero2020deep}
{\sc C.~Cuchiero, M.~Larsson, and J.~Teichmann}, {\em Deep neural networks, generic universal interpolation, and controlled odes}, SIAM Journal on Mathematics of Data Science, 2 (2020), pp.~901--919, \url{https://doi.org/10.1137/19M1284117}.

\bibitem{cybenko1989approximation}
{\sc G.~Cybenko}, {\em Approximation by superpositions of a sigmoidal function}, Mathematics of Control, Signals and Systems, 2 (1989), pp.~303--314.

\bibitem{duan2025minimal}
{\sc Y.~Duan and Y.~Cai}, {\em A minimal control family of dynamical systems for universal approximation}, IEEE Transactions on Automatic Control, 70 (2025), pp.~7233--7244, \url{https://doi.org/10.1109/TAC.2025.3569123}.

\bibitem{e2017proposal}
{\sc W.~E}, {\em A proposal on machine learning via dynamical systems}, Communications in Mathematics and Statistics, 5 (2017), pp.~1--11.

\bibitem{e2019barron}
{\sc W.~E, C.~Ma, and L.~Wu}, {\em Barron spaces and the compositional function spaces for neural network models}, arXiv preprint arXiv:1906.08039,  (2019).

\bibitem{e2019priori}
{\sc W.~E, C.~Ma, and L.~Wu}, {\em A priori estimates of the population risk for two-layer neural networks}, Communications in Mathematical Sciences, 17 (2019), pp.~1407--1425.

\bibitem{fang2025two}
{\sc Q.~Fang, L.~Shi, M.~Xu, and D.-X. Zhou}, {\em Two-dimensional deep relu cnn approximation for korobov functions: A constructive approach}, arXiv preprint arXiv:2503.07976,  (2025).

\bibitem{gao2024adaptive}
{\sc Z.~Gao, L.~Yan, and T.~Zhou}, {\em Adaptive operator learning for infinite-dimensional bayesian inverse problems}, SIAM/ASA Journal on Uncertainty Quantification, 12 (2024), pp.~1389--1423.

\bibitem{geng2025mean}
{\sc Z.~Geng, M.~Deng, X.~Bai, J.~Z. Kolter, and K.~He}, {\em Mean flows for one-step generative modeling}, in The Thirty-ninth Annual Conference on Neural Information Processing Systems, 2025, \url{https://openreview.net/forum?id=uWj4s7rMnR}.

\bibitem{guo2022normalizing}
{\sc L.~Guo, H.~Wu, and T.~Zhou}, {\em Normalizing field flows: Solving forward and inverse stochastic differential equations using physics-informed flow models}, Journal of Computational Physics, 461 (2022), p.~111202.

\bibitem{haber2018learning}
{\sc E.~Haber, L.~Ruthotto, E.~Holtham, and S.-H. Jun}, {\em Learning across scales---multiscale methods for convolution neural networks}, in Proceedings of the AAAI conference on artificial intelligence, vol.~32, 2018.

\bibitem{he2022approximation}
{\sc J.~He, L.~Li, and J.~Xu}, {\em Approximation properties of deep relu cnns}, Research in the mathematical sciences, 9 (2022), p.~38.

\bibitem{he2022relu}
{\sc J.~He, L.~Li, and J.~Xu}, {\em Relu deep neural networks from the hierarchical basis perspective}, Computers \& Mathematics with Applications, 120 (2022), pp.~105--114.

\bibitem{he2020relu}
{\sc J.~He, L.~Li, J.~Xu, and C.~Zheng}, {\em Relu deep neural networks and linear finite elements}, Journal of Computational Mathematics, 38 (2020), pp.~502--527.

\bibitem{he2024mgno}
{\sc J.~He, X.~Liu, and J.~Xu}, {\em Mg{NO}: Efficient parameterization of linear operators via multigrid}, in The Twelfth International Conference on Learning Representations, 2024, \url{https://openreview.net/forum?id=8OxL034uEr}.

\bibitem{he2025self}
{\sc J.~He, X.~Liu, and J.~Xu}, {\em Self-composing neural operators with depth and accuracy scaling via adaptive train-and-unroll approach}, arXiv preprint arXiv:2508.20650,  (2025).

\bibitem{he2023expressivity}
{\sc J.~He, T.~Mao, and J.~Xu}, {\em Expressivity and approximation properties of deep neural networks with relu$^k$ activation}, arXiv preprint arXiv:2312.16483,  (2023).

\bibitem{he2023deep}
{\sc J.~He and J.~Xu}, {\em Deep neural networks and finite elements of any order on arbitrary dimensions}, arXiv preprint arXiv:2312.14276,  (2023).

\bibitem{he2016deep}
{\sc K.~He, X.~Zhang, S.~Ren, and J.~Sun}, {\em Deep residual learning for image recognition}, in Proceedings of the IEEE conference on computer vision and pattern recognition, 2016, pp.~770--778.

\bibitem{hornik1989multilayer}
{\sc K.~Hornik, M.~Stinchcombe, and H.~White}, {\em Multilayer feedforward networks are universal approximators}, Neural Networks, 2 (1989), pp.~359--366.

\bibitem{Jia2020}
{\sc F.~Jia, J.~Liu, and X.~C. Tai}, {\em {A regularized convolutional neural network for semantic image segmentation}}, Analysis and Applications,  (2020), pp.~1--19, \url{https://doi.org/10.1142/S0219530519410148}, \url{https://arxiv.org/abs/1907.05287}.

\bibitem{masato2025universal}
{\sc M.~Kimura, K.~Matsui, and Y.~Mizuno}, {\em Universal approximation property of odenet and resnet with a single activation function}, Journal of Computational Mathematics and Data Science, 15 (2025), p.~100116.

\bibitem{klusowski2018approximation}
{\sc J.~M. Klusowski and A.~R. Barron}, {\em Approximation by combinations of relu and squared relu ridge functions with $\ell^1$ and $\ell^0$ controls}, IEEE Transactions on Information Theory, 64 (2018), pp.~7649--7656.

\bibitem{kovachki2021universal}
{\sc N.~Kovachki, S.~Lanthaler, and S.~Mishra}, {\em On universal approximation and error bounds for fourier neural operators}, Journal of Machine Learning Research, 22 (2021), pp.~1--76.

\bibitem{kovachki2023neural}
{\sc N.~Kovachki, Z.~Li, B.~Liu, K.~Azizzadenesheli, K.~Bhattacharya, A.~Stuart, and A.~Anandkumar}, {\em Neural operator: Learning maps between function spaces with applications to pdes}, Journal of Machine Learning Research, 24 (2023), pp.~1--97.

\bibitem{krizhevsky2012imagenet}
{\sc A.~Krizhevsky, I.~Sutskever, and G.~E. Hinton}, {\em Imagenet classification with deep convolutional neural networks}, Advances in neural information processing systems, 25 (2012).

\bibitem{lanthaler2022error}
{\sc S.~Lanthaler, S.~Mishra, and G.~E. Karniadakis}, {\em Error estimates for deeponets: A deep learning framework in infinite dimensions}, Transactions of Mathematics and its Applications, 6 (2022), p.~tnac001.

\bibitem{lecun2015deep}
{\sc Y.~LeCun, Y.~Bengio, and G.~Hinton}, {\em Deep learning}, nature, 521 (2015), pp.~436--444.

\bibitem{lecun1998gradient}
{\sc Y.~Lecun, L.~Bottou, Y.~Bengio, and P.~Haffner}, {\em Gradient-based learning applied to document recognition}, Proceedings of the IEEE, 86 (1998), pp.~2278--2324.

\bibitem{leshno1993multilayer}
{\sc M.~Leshno, V.~Y. Lin, A.~Pinkus, and S.~Schocken}, {\em Multilayer feedforward networks with a nonpolynomial activation function can approximate any function}, Neural networks, 6 (1993), pp.~861--867.

\bibitem{li2024approximation}
{\sc J.~Li, H.~Feng, and D.-X. Zhou}, {\em Approximation analysis of cnns from a feature extraction view}, Analysis and Applications, 22 (2024), pp.~635--654.

\bibitem{li2023minimum}
{\sc L.~Li, Y.~Duan, G.~Ji, and Y.~Cai}, {\em Minimum width of leaky-relu neural networks for uniform universal approximation}, in International Conference on Machine Learning, PMLR, 2023, pp.~19460--19470.

\bibitem{li2023deep}
{\sc Q.~Li, T.~Lin, and Z.~Shen}, {\em Deep learning via dynamical systems: An approximation perspective.}, J. Eur. Math. Soc. 25, 5 (2023), p.~1671–1709, \url{https://doi.org/10.4171/JEMS/1221}.

\bibitem{li2021fourier}
{\sc Z.~Li, N.~B. Kovachki, K.~Azizzadenesheli, B.~liu, K.~Bhattacharya, A.~Stuart, and A.~Anandkumar}, {\em Fourier neural operator for parametric partial differential equations}, in International Conference on Learning Representations, 2021, \url{https://openreview.net/forum?id=c8P9NQVtmnO}.

\bibitem{li2026deep}
{\sc Z.~Li, K.~Liu, Y.~Song, H.~Yue, and E.~Zuazua}, {\em Deep neural ode operator networks for pdes}, Mathematical Models and Methods in Applied Sciences,  (2026).

\bibitem{li2017flow}
{\sc Z.~Li and Z.~Shi}, {\em A flow model of neural networks}, arXiv preprint arXiv:1708.06257,  (2017).

\bibitem{lin2018resnet}
{\sc H.~Lin and S.~Jegelka}, {\em Resnet with one-neuron hidden layers is a universal approximator}, Advances in neural information processing systems, 31 (2018).

\bibitem{lipman2023flow}
{\sc Y.~Lipman, R.~T.~Q. Chen, H.~Ben-Hamu, M.~Nickel, and M.~Le}, {\em Flow matching for generative modeling}, in The Eleventh International Conference on Learning Representations, 2023, \url{https://openreview.net/forum?id=PqvMRDCJT9t}.

\bibitem{liu2024characterizing}
{\sc C.~Liu, E.~Liang, and M.~Chen}, {\em Characterizing resnet's universal approximation capability}, in Forty-first International Conference on Machine Learning, 2024, \url{https://openreview.net/forum?id=z7zHsNFXHc}.

\bibitem{liu2025integral}
{\sc X.~Liu, T.~Mao, and J.~Xu}, {\em Integral representations of sobolev spaces via relu$^k$ activation function and optimal error estimates for linearized networks}, arXiv preprint arXiv:2505.00351,  (2025).

\bibitem{liu2024mitigating}
{\sc X.~Liu, B.~Xu, S.~Cao, and L.~Zhang}, {\em Mitigating spectral bias for the multiscale operator learning}, Journal of Computational Physics, 506 (2024), p.~112944.

\bibitem{lu2021learning}
{\sc L.~Lu, P.~Jin, G.~Pang, Z.~Zhang, and G.~E. Karniadakis}, {\em Learning nonlinear operators via deeponet based on the universal approximation theorem of operators}, Nature machine intelligence, 3 (2021), pp.~218--229.

\bibitem{lu2018beyond}
{\sc Y.~Lu, A.~Zhong, Q.~Li, and B.~Dong}, {\em Beyond finite layer neural networks: Bridging deep architectures and numerical differential equations}, in International conference on machine learning, PMLR, 2018, pp.~3276--3285.

\bibitem{lu2017expressive}
{\sc Z.~Lu, H.~Pu, F.~Wang, Z.~Hu, and L.~Wang}, {\em The expressive power of neural networks: A view from the width}, Advances in neural information processing systems, 30 (2017).

\bibitem{markowich2024pde}
{\sc P.~Markowich and S.~Portaro}, {\em Pde models for deep neural networks: Learning theory, calculus of variations and optimal control}, arXiv preprint arXiv:2411.06290,  (2024), \url{https://arxiv.org/abs/2411.06290}.

\bibitem{mhaskar1993approximation}
{\sc H.~N. Mhaskar}, {\em Approximation properties of a multilayered feedforward artificial neural network}, Advances in Computational Mathematics, 1 (1993), pp.~61--80.

\bibitem{mhaskar1995degree}
{\sc H.~N. Mhaskar and C.~A. Micchelli}, {\em Degree of approximation by neural and translation networks with a single hidden layer}, Advances in applied mathematics, 16 (1995), pp.~151--183.

\bibitem{mhaskar2016deep}
{\sc H.~N. Mhaskar and T.~Poggio}, {\em Deep vs. shallow networks: An approximation theory perspective}, Analysis and Applications, 14 (2016), pp.~829--848.

\bibitem{molinaro2023neural}
{\sc R.~Molinaro, Y.~Yang, B.~Engquist, and S.~Mishra}, {\em Neural inverse operators for solving pde inverse problems}, in Proceedings of the 40th International Conference on Machine Learning, vol.~202 of Proceedings of Machine Learning Research, PMLR, 2023, pp.~25105--25139.

\bibitem{montufar2014number}
{\sc G.~Mont{\'u}far, R.~Pascanu, K.~Cho, and Y.~Bengio}, {\em On the number of linear regions of deep neural networks}, Advances in neural information processing systems, 27 (2014).

\bibitem{oono2019approximation}
{\sc K.~Oono and T.~Suzuki}, {\em Approximation and non-parametric estimation of {ResNet}-type convolutional neural networks}, in International Conference on Machine Learning, PMLR, 2019, pp.~4922--4931.

\bibitem{opschoor2020deep}
{\sc J.~A. Opschoor, P.~C. Petersen, and C.~Schwab}, {\em Deep relu networks and high-order finite element methods}, Analysis and Applications, 18 (2020), pp.~715--770.

\bibitem{petersen2020equivalence}
{\sc P.~Petersen and F.~Voigtlaender}, {\em Equivalence of approximation by convolutional neural networks and fully-connected networks}, Proceedings of the American Mathematical Society, 148 (2020), pp.~1689--1704.

\bibitem{raonic2023convolutional}
{\sc B.~Raonic, R.~Molinaro, T.~Rohner, S.~Mishra, and E.~de~Bezenac}, {\em Convolutional neural operators}, in ICLR 2023 workshop on physics for machine learning, 2023.

\bibitem{ruizbalet2023neuralode}
{\sc D.~Ruiz-Balet and E.~Zuazua}, {\em Neural {ODE} control for classification, approximation and transport}, SIAM Review, 65 (2023), pp.~735--773, \url{https://doi.org/10.1137/21M1411433}, \url{https://doi.org/10.1137/21M1411433}.

\bibitem{shen2019deep}
{\sc Z.~Shen, H.~Yang, and S.~Zhang}, {\em Deep network approximation characterized by number of neurons}, arXiv preprint arXiv:1906.05497,  (2019).

\bibitem{shen2022optimal}
{\sc Z.~Shen, H.~Yang, and S.~Zhang}, {\em Optimal approximation rate of {ReLU} networks in terms of width and depth}, Journal de Math{\'e}matiques Pures et Appliqu{\'e}es, 157 (2022), pp.~101--135.

\bibitem{siegel2023optimal}
{\sc J.~W. Siegel}, {\em Optimal approximation rates for deep relu neural networks on sobolev and besov spaces}, Journal of Machine Learning Research, 24 (2023), pp.~1--52.

\bibitem{siegel2024sharp}
{\sc J.~W. Siegel and J.~Xu}, {\em Sharp bounds on the approximation rates, metric entropy, and n-widths of shallow neural networks}, Foundations of Computational Mathematics, 24 (2024), pp.~481--537.

\bibitem{tabuada2024universal}
{\sc P.~Tabuada and B.~Gharesifard}, {\em Universal approximation power of deep neural networks via nonlinear control theory}, arXiv preprint arXiv:2007.06007,  (2024), \url{https://arxiv.org/abs/2007.06007}.

\bibitem{tai2024pottsmgnet}
{\sc X.-C. Tai, H.~Liu, and R.~Chan}, {\em Pottsmgnet: A mathematical explanation of encoder-decoder based neural networks}, SIAM Journal on Imaging Sciences, 17 (2024), pp.~540--594.

\bibitem{tai2025mathematical}
{\sc X.-C. TAI, H.~LIU, R.~H. CHAN, and L.~LI}, {\em A mathematical explanation of unet}, Mathematical Foundations of Computing, 8 (2025), pp.~874--889.

\bibitem{tai2025mathematical1}
{\sc X.-C. Tai, H.~Liu, L.~Li, and R.~H. Chan}, {\em A mathematical explanation of transformers for large language models and gpts}, arXiv preprint arXiv:2510.03989,  (2025).

\bibitem{telgarsky2016benefits}
{\sc M.~Telgarsky}, {\em Benefits of depth in neural networks}, in Proceedings of the 29th Annual Conference on Learning Theory, PMLR, 2016, pp.~1517--1539.

\bibitem{wu2024transolver}
{\sc H.~Wu, H.~Luo, H.~Wang, J.~Wang, and M.~Long}, {\em Transolver: A fast transformer solver for pdes on general geometries}, in International Conference on Machine Learning, PMLR, 2024, pp.~53681--53705.

\bibitem{yang2025deep}
{\sc Y.~Yang and J.~He}, {\em Deep neural networks with general activations: Super-convergence in sobolev norms}, arXiv preprint arXiv:2508.05141,  (2025).

\bibitem{yang2023nearly}
{\sc Y.~Yang, H.~Yang, and Y.~Xiang}, {\em Nearly optimal vc-dimension and pseudo-dimension bounds for deep neural network derivatives}, Advances in Neural Information Processing Systems, 36 (2023), pp.~21721--21756.

\bibitem{yarotsky2017error}
{\sc D.~Yarotsky}, {\em Error bounds for approximations with deep relu networks}, Neural Networks, 94 (2017), pp.~103--114.

\bibitem{zhou2020universality}
{\sc D.-X. Zhou}, {\em Universality of deep convolutional neural networks}, Applied and Computational Harmonic Analysis, 48 (2020), pp.~787--794.

\end{thebibliography}

\appendix
 \section{Proof of Lemma~\ref{thm:UAP_general_dim}}\label{sec:appendix}
\begin{proof}
The universal approximation property for the separation scheme follows directly from known results, while the composition structured flow requires a more elaborate construction. Let $d=\max(2d_x+1,d_y)$.

\noindent\textbf{1. Separation structured flow.}  
\cite{duan2025minimal} shows that flows of piecewise‑constant vector fields taking values in
\(\mathcal{F}(\sigma)=\{z\mapsto Az+b\}\cup\{\sigma\}\) are universal approximators, with suitable linear maps \(P,R\), provided the dimension $D\ge d$.  
By choosing the coefficient \(W_t,b_t,\alpha_t\) alternate between zero and non-zero, the separation‑structured flow
\( W_t z + b_t + \alpha_t\sigma_a(z)\) can include the above class of vector fields.  
Hence, it inherits the universal approximation property.

\noindent\textbf{2. Composition structure.}  
\cite{li2023deep,duan2025minimal} prove that the flow of
\begin{equation}\label{eq:known-UAP}
\frac{dz}{dt} = D_t\sigma_a(A_t z + b_t),
\end{equation}
with piecewise‑constant diagonal \(D_t\) (entries \(\pm1\)), is universal approximator when $D\ge d$.  
We construct a double‑width system on \(\mathbb R^{2d}\) that simulates \eqref{eq:known-UAP}.  
Write \(D_t=\operatorname{diag}(D^1_t,\dots,D^d_t)\), \(A_t=[\alpha^1_t;\dots;\alpha^d_t]\), \(b_t=[b^1_t;\dots;b^d_t]\).  
For each \(i\), define
\[
\mathcal{M}(D^i_t,\alpha^i_t)=\begin{cases}
\bigl[\begin{smallmatrix}\alpha^i_t&-\alpha^i_t\\0&0\end{smallmatrix}\bigr],&D^i_t=1,\\
\bigl[\begin{smallmatrix}0&0\\\alpha^i_t&-\alpha^i_t\end{smallmatrix}\bigr],&D^i_t=-1,
\end{cases}
\qquad
\mathcal{N}(D^i_t,b^i_t)=\begin{cases}
\bigl[\begin{smallmatrix}b^i_t\\0\end{smallmatrix}\bigr],&D^i_t=1,\\
\bigl[\begin{smallmatrix}0\\b^i_t\end{smallmatrix}\bigr],&D^i_t=-1.
\end{cases}
\]
Set \(\tilde A_t=[\mathcal{M}(D^1_t,\alpha^1_t);\dots;\mathcal{M}(D^d_t,\alpha^d_t)]\), \(\tilde b_t=[\mathcal{N}(D^1_t,b^1_t);\dots;\mathcal{N}(D^d_t,b^d_t)]\) and consider
\begin{equation}\label{eq:double-width}
\begin{cases}
  \frac{d\hat z}{dt} = \sigma_a(\tilde A_t\hat z + \tilde b_t),\\[4pt]
\hat z(0)=\frac{1}{1+a}\bigl(\sigma_a(z(0)),\sigma_a(-z(0))\bigr)^\top .
\end{cases}
\end{equation}
A direct calculation (using \(\sigma_a(z)-\sigma_a(-z)=(1+a)z\)) shows that \(\tilde z(t):=(I_d,-I_d)\hat z(t)\) satisfies \eqref{eq:known-UAP} with \(\tilde z(0)=z(0)\).  
Therefore, the flow $G_{\theta_t}^T$ of \eqref{eq:double-width} reproduces the universal approximator \eqref{eq:known-UAP}.
That is, there exist two linear maps $P_1:\mathbb R^{d_x}\to\mathbb R^{d},R_1:\mathbb R^{d}\to\mathbb R^{d_y}$ such that 
\begin{equation*}
\|R_1\circ S^\top\circ G_{\theta_t}^T\circ\frac{\sigma_a}{1+a}\circ S\circ P_1-F\|_{L^\infty(\Omega)}\le \varepsilon.
\end{equation*}
where $S=(I_d,-I_d)^\top$. Since the leaky-ReLU $\sigma_a$ can actually be represented as the flow $H^\tau$ of a system with $\sigma_a$ at the right-hand side:

\begin{equation*}
    \frac{\mathrm d z(t)}{\mathrm d t} =  \sigma_a\left(z(t)\right), \quad 0 < t \le \tau,
\end{equation*}
where $\tau=\frac{ln(a)}{a-1}$. 
Let $R=R_1\circ S^\top$, $P=\frac{e^{-\tau}}{1+a} S\circ P_1$, $T'=T+\tau$ and $F^{T'}_{\theta'_t}=G_{\theta_t}^T\circ H^\tau$, then
\begin{equation*}
\|R\circ F^{T'}_{\theta'_t}\circ P-F\|_{L^\infty(\Omega)}\le \varepsilon.
\end{equation*}
This completes the proof for both architectures.
\end{proof}

\end{document}